\documentclass[10pt]{article}

\usepackage{natbib}
\usepackage[verbose=true,letterpaper]{geometry}
\makeatletter
\newcommand{\@toptitlebar}{
  \hrule height 4\p@
  \vskip 0.25in
  \vskip -\parskip%
}
\newcommand{\@bottomtitlebar}{
  \vskip 0.29in
  \vskip -\parskip
  \hrule height 1\p@
  \vskip 0.09in%
}
\providecommand{\@maketitle}{}
\renewcommand{\@maketitle}{%
  \vbox{%
    \hsize\textwidth
    \linewidth\hsize
    \vskip 0.1in
    \@toptitlebar
    \centering
    {\LARGE\bf \@title\par}
    \@bottomtitlebar
    \def\And{%
        \end{tabular}\hfil\linebreak[0]\hfil%
        \begin{tabular}[t]{c}\bf\rule{\z@}{24\p@}\ignorespaces%
    }
    \def\AND{%
        \end{tabular}\hfil\linebreak[4]\hfil%
        \begin{tabular}[t]{c}\bf\rule{\z@}{24\p@}\ignorespaces%
    }
    \begin{minipage}[c]{\textwidth}%
        \centering%
        \begin{tabular}[t]{c}\bf\rule{\z@}{24\p@}%
        \@author
        \end{tabular}%
    \end{minipage}
    \vskip 0.3in \@minus 0.1in
  }
}
\makeatother

\usepackage[theorems]{macros}
\usepackage[rawfloats]{floatrow}
\newfloatcommand{capbtabbox}{table}[][\FBwidth]
\usepackage{wrapfig}
\usepackage{rotating}
\usepackage{algorithm}
\usepackage{algorithmic}

\usepackage[usenames,dvipsnames]{xcolor}
\usepackage[font=footnotesize]{subfig}
\usepackage{tikz, pgfplots}
\definecolor{Base00} {HTML}{ffffff} 
\definecolor{Base01} {HTML}{e0e0e0} 
\definecolor{Base02} {HTML}{d6d6d6} 
\definecolor{Base03} {HTML}{8e908c} 
\definecolor{Base04} {HTML}{969896} 
\definecolor{Base05} {HTML}{4d4d4c} 
\definecolor{Base06} {HTML}{282a2e} 
\definecolor{Base07} {HTML}{1d1f21} 
\definecolor{Red}    {HTML}{c82829} 
\definecolor{Base09} {HTML}{f5871f} 
\definecolor{Yellow} {HTML}{eab700} 
\definecolor{Green}  {HTML}{718c00} 
\definecolor{Cyan}   {HTML}{3e999f} 
\definecolor{Blue}   {HTML}{4271ae} 
\definecolor{Magenta}{HTML}{8959a8} 
\definecolor{Base0F} {HTML}{a3685a} 

\colorlet{LightBase00}{Base00!20}
\colorlet{LightBase01}{Base01!20}
\colorlet{LightBase02}{Base02!20}
\colorlet{LightBase03}{Base03!20}
\colorlet{LightBase04}{Base04!20}
\colorlet{LightBase05}{Base05!20}
\colorlet{LightBase06}{Base06!20}
\colorlet{LightBase07}{Base07!20}
\colorlet{LightRed}{Red!20}
\colorlet{LightBase09}{Base09!20}
\colorlet{LightYellow}{Yellow!20}
\colorlet{LightGreen}{Green!20}
\colorlet{LightCyan}{Cyan!20}
\colorlet{LightBlue}{Blue!20}
\colorlet{LightMagenta}{Magenta!20}
\colorlet{LightBase0F}{Base0F!20}  

\colorlet{DarkBase00}{Base00!90!black}
\colorlet{DarkBase01}{Base01!90!black}
\colorlet{DarkBase02}{Base02!90!black}
\colorlet{DarkBase03}{Base03!90!black}
\colorlet{DarkBase04}{Base04!90!black}
\colorlet{DarkBase05}{Base05!90!black}
\colorlet{DarkBase06}{Base06!90!black}
\colorlet{DarkBase07}{Base07!90!black}
\colorlet{DarkRed}{Red!90!black}
\colorlet{DarkBase09}{Base09!90!black}
\colorlet{DarkYellow}{Yellow!90!black}
\colorlet{DarkGreen}{Green!90!black}
\colorlet{DarkCyan}{Cyan!90!black}
\colorlet{DarkBlue}{Blue!90!black}
\colorlet{DarkMagenta}{Magenta!90!black}
\colorlet{DarkBase0F}{Base0F!90!black}

\definecolor{plots_01}{rgb}{0.0,0.605603,0.97868}
\definecolor{plots_02}{rgb}{0.888874,0.435649,0.278123}
\definecolor{plots_03}{rgb}{0.242224,0.643275,0.304449}
\definecolor{plots_04}{rgb}{0.76444,0.444112,0.824298}
\definecolor{plots_05}{rgb}{0.675544,0.555662,0.0942343}
\definecolor{plots_06}{rgb}{4.82118e-7,0.665759,0.680997}
\definecolor{plots_07}{rgb}{0.930767,0.367477,0.57577}
\definecolor{plots_08}{rgb}{0.776982,0.509743,0.146425}
\definecolor{plots_09}{rgb}{3.80773e-7,0.664268,0.552951}
\definecolor{plots_10}{rgb}{0.558465,0.593485,0.117481}
\definecolor{plots_11}{rgb}{5.94762e-7,0.660879,0.798179}
\definecolor{plots_12}{rgb}{0.609671,0.499185,0.911781}
\definecolor{plots_13}{rgb}{0.380002,0.551053,0.966506}
\definecolor{plots_14}{rgb}{0.942182,0.375164,0.451817}
\definecolor{plots_15}{rgb}{0.868402,0.395989,0.713515}
\definecolor{plots_16}{rgb}{0.423147,0.622495,0.198771}
\definecolor{plots_17}{rgb}{0.444437,0.549958,0.415537}

\usepackage[inline]{enumitem}
\setlist*[enumerate,1]{label={\arabic*)}}

\hypersetup{unicode=true, colorlinks=true, linkcolor=DarkBlue, citecolor=DarkBlue, urlcolor=DarkBlue}
\usetikzlibrary{positioning, backgrounds, external, intersections}
\usepgfplotslibrary{fillbetween, groupplots}

\tikzset{every node/.append style={font=\footnotesize}}
\pgfplotsset{compat=newest}
\pgfplotsset{
    tickpos=left,
    ytick align=outside,
    xtick align=outside,
    wla_diagrams/.style={
        clip mode=individual,
        samples=100,
        enlargelimits=false,
        domain = -3.5:5,
        xmin=-3.5, xmax=5,
        ymin=-0.3, ymax=0.45,
        height=4.5cm, width=6cm,
        axis lines=middle,
        x axis line style=-,
        y axis line style={draw=none},
        xtick = {0}, ytick = {0},
        enlargelimits=false, clip=false, 
        axis on top,
        after end axis/.code={
            \path (0,0) node [anchor=north west,yshift=-0.075cm] {} node [anchor=south east,xshift=-0.075cm] {};
            }
    },
    julia_contour_plot/.style={
        width=4.5cm,      height=4.5cm,
        axis on top,    enlargelimits=false,
        clip bounding box=upper bound,
        xticklabels={}, yticklabels={},
        xlabel style={overlay}, ylabel style={overlay},
        xticklabel style={overlay}, yticklabel style={overlay},
    },
    julia_convergence_plot/.style={
        name=axis,
        axis on top,
        ymajorgrids,
        enlargelimits=false,
        width=7cm, height=4cm,
        every tick label/.append style={font=\tiny},
        major grid style={line width=.2pt,draw=gray!50},
        x label style={font=\tiny},
        y label style={font=\tiny},
        legend  style={font=\tiny, nodes={scale=0.75, transform shape}, fill opacity=0.6, text opacity =1},
        legend cell align={left},
        legend image code/.code={
            \draw[mark repeat=2,mark phase=2, thick]
            plot coordinates {
            (0cm,0cm)
            (0.1cm,0cm)        
            (0.2cm,0cm)        
            };%
            },
    },
    julia_violin_plot/.style={
        name=axis,
        axis on top,
        xmajorgrids,
        enlargelimits=false,
        width=5cm, height=4.5cm,
        major grid style={line width=.2pt,draw=gray!50},
        x label style={font=\tiny},
        y label style={font=\tiny},
    },
}
    
\newtheorem{assumption}{Assumption}
\newlist{paperlayout}{itemize*}{1}
\setlist*[paperlayout,1]{%
  label={},
  before={},
  itemjoin={;\;}, 
  itemjoin*={;\; and finally,},
  after={}
}

\def\papertitle{Boosted Density Estimation Remastered} 

\title{\papertitle}

\newcommand{\customfootnotetext}[2]{{
  \renewcommand{\thefootnote}{#1}
  \footnotetext[0]{#2}}}
\author{
    Zac Cranko$^{*\dagger}$
    \And
    Richard Nock$^{*\dagger\ddagger}$
    \end{tabular}\\[0.3in]
    \begin{tabular}[t]{c}
        \texttt{firstname.lastname@data61.csiro.edu.au}
}
\def\radon{\mathcaleu{R}(\cal X)}
\def\prob{\mathcaleu{D}(\cal X)}
\def\disc{\mathcaleu{C}(\cal X)}

\begin{document}
\newgeometry{
    textheight=9in,
    textwidth=5.5in,
    top=1in,
    headheight=12pt,
    headsep=25pt,
    footskip=30pt
}
\maketitle
\customfootnotetext{$*$}{Data61, CSIRO}
\customfootnotetext{$\dagger$}{The Australian National University}
\customfootnotetext{$\ddagger$}{The University of Sydney}

\begin{abstract}
    There has recently been a steady increase in the number iterative approaches to density estimation. However, an accompanying burst of formal convergence guarantees has not followed; all results pay the price of heavy assumptions which are often unrealistic or hard to check. The \emph{Generative Adversarial Network (GAN)} literature --- seemingly orthogonal to the aforementioned pursuit --- has had the side effect of a renewed interest in variational divergence minimisation (notably $f$-GAN). We show that by introducing a \textit{weak learning assumption} (in the sense of the classical boosting framework) we are able to import some recent results from the GAN literature to develop an iterative boosted density estimation algorithm, including formal convergence results with rates, that does not suffer the shortcomings other approaches. We show that the density fit is an exponential family, and as part of our analysis obtain an improved variational characterization of $f$-GAN.
\end{abstract}

\section{Introduction}

In the emerging area of \emph{Generative Adversarial Networks (GAN's)}
\citep{goodfellow2014generative} a binary classifier (called a
\emph{discriminator} in the parlance of the GAN literature) is used
learn a highly efficient sampler for a data distribution $P$;
combining what would traditionally be two steps --- first learning the
density function from a family of densities, then fine-tuning a
sampler --- into one. Interest in this field as sparked a series of
formal inquiries and generalisations describing GAN's in terms of
(among other things) divergence minimisation
\citep{nowozin2016f,arjovsky2017wasserstein}. Using a similar
framework to \citet{nowozin2016f}, \citet{grover2017boosted} make a
preliminary analysis of an algorithm that takes a series of
iteratively trained discriminators to estimate a densityls
function\footnote{\citet{grover2017boosted} call this procedure
  ``multiplicative discriminative boosting''.}. The cost here, insofar
as we have been able to devise, is that one forgoes learning an
efficient sampler (as with a GAN), and must make do with classical
sampling techniques to sample from the learned density. This we leave
the issue of efficient sampling in large dimensions as an open problem, and instead focus on analysing the densities learned with formal convergence guarantees. 

\begin{paperlayout}[before={\vspace{\baselineskip}The rest of the paper and our contributions are as follows:\;}]
    \item in \autoref{sec:preliminaries}, to make explicit the connections between classification, density estimation, and divergence minimisation  we re-introduce the variational $f$-divergence formulation, and in doing so are able to fully explain some of the underspecified components of $f$-GAN \citep{nowozin2016f}
    \item in \autoref{sec:boosted_density_estimation}, we relax a number of the assumptions of \citet{grover2017boosted}, and then give both more general, and much stronger bounds for their algorithm
    \item in \autoref{sec:experiments}, we apply our algorithm to several toy datasets in order demonstrate convergence and compare directly with \citet{tgbssAB}
    \item a final section \autoref{sec:conclusion} concludes%
    .
\end{paperlayout}
\begin{paperlayout}[before ={The appendices that follow in the supplementary material are:\;}]
    \item \autoref{sec:epilogue}, we compare our formal results with other related works
    \item \autoref{sec:the_error_term}, a geometric account of the function class in the variational form of an $f$-divergence
    \item \autoref{sec:boosting_with_estimates}, a further relaxation
      of the weak learning assumptions to some that could actually be
      estimated experimentally and a proof that the boosting rates are
      slightly worse but of essentially the same order
    \item \autoref{sec:allproofs}, proofs for the main formal results from the paper
    \item \autoref{sec:experimental_setup}, technical details for the settings of our experiments%
    .
\end{paperlayout}

\subsection{Related work}\label{sec:related}

In learning a density --- that is minimising a divergence $\min_Q I(P,Q) $---it is remarkable that most previous approaches \citep[and references therein]{grover2017boosted,gwfbdBV,lbMD,mfaVB,rsBD,tgbssAB,zSG} have investigated a single update rule, not unlike Frank--Wolfe optimisation:
\begin{gather}
    \cramped{\forall{\alpha_t \in [0,1]} 
    Q_{t+1} = f(\alpha_t g(c_t) + (1-\alpha_t) g(Q_t)),}
    \hphantom{(2)}
    \label{eq:genupdatebden}
\end{gather}
where $g$ is in general (but not always) the identity. \citet{grover2017boosted} is one (recent) rare exception to \eqref{eq:genupdatebden} wherein alternative choices are explored. Few works in this area are accompanied by convergence proofs, and even less display convergence \textit{rates}, which are mandatory in a boosting setting \citep{gwfbdBV,lbMD,rsBD,tgbssAB,zSG}.


To establish convergence and/or bound convergence by a rate, \textit{all} approaches necessarily make structural assumptions or approximations on the parameters involved in \eqref{eq:genupdatebden}. These can be on the (local) variation of $I$ \citep{gwfbdBV,neDE,zSG}; the true density $P$ or the updates $Q_t$ \citep{grover2017boosted,gwfbdBV,lbMD}; the mixing parameter $\alpha_t$ \citep{mfaVB,tgbssAB}; the weak learners $c_t$ \citep{grover2017boosted,rsBD,tgbssAB,zSG}; the previous updates, $(c_j)_{j\leq t}$ ,\citep{rsBD}; and so on. The price to get the best geometric convergence bounds is in fact heavy considering that the update $c_t$ is in all cases required to be close to the optimal one \citep[Corollaries 2, 3]{tgbssAB}. 

However, it must be kept in mind that for many of these works
\cite[viz.][]{tgbssAB} the primary objective is to develop an
efficient black box sampler for $P$, in particular for large dimensions. Our objective however is to focus on furtive lack of formal results on the densities and convergence, instead leaving the problem of sampling from these densities as an open question.

\section{Preliminaries}\label{sec:preliminaries}
In the sequel $(\cal X, \tau)$ is a topological space. Unnormalised Borel measures on $\cal X$ are indicated by decorated capital letters, $\tilde P$, and Borel probability measures by capital letters without decoration, $P$. To a function $f:\cal X\to \Rx$ we associate another function $f^*$, called the \emph{Fenchel conjugate} with $f^*(x^*)\defas \sup_{x\in\cal X}\inp{x^*, x} - f(x)$. If $f$ is convex, proper, and lower semi-continuous, $f = (f^*)^*$. If $f$ is strictly convex and differentiable on $\intr(\dom f)$ then $(f^*)' = (f')^{-1}$. Theorem-like formal statements are numbered to be logically consistent with their appearance in the appendix (\autoref{sec:allproofs}) to which we defer all proofs.

\begin{table}[t]
    \caption{Some common $f$-divergences and their variational components.\label{tab:common_divergences}}
    \centering
    \scriptsize
    \makebox[\textwidth][c]{
        \begin{tabular}{lccccc}\toprule
            &\multicolumn{1}{c}{$\I_f$}      & \multicolumn{1}{c}{$f(t)$}                                   & \multicolumn{1}{c}{$f^*(t^*)$}                                                                                                         & \multicolumn{1}{c}{$f'(t)$}                  & \multicolumn{1}{c}{$(f^*\circ f')(t)$}\\\midrule
            Kullback--Liebler          &$\kl$                           & $t\log t    $                                                & $\exp\rbr{t^*-1}$                                                                                                                           & $\log t + 1$                                     & $t$\\
            Reverse KL                 &$\rkl$                          & $\abs{t-1}$                                                  & $-\log( -t^*)-1$                                                                                                                       & $-1/t$                                       & $\log t-1$\\
            Hellinger                  & -                              & $(\sqrt{t}-1)^2$                                             & $3(t^*-1)\inv-1$                                                                                                                       & $1-1/t$                                      & $\sqrt{t}-1$\\
            Pearson                    &$\chi^2$                        & $(t-1)^2$                                                    & $t^*(4 + t^*)/4$                                                                                                                & $2(t-1)$                                     & $t^2-1$\\
            GAN                        &$\textrm{GAN}$                  & $t \log t-(t+1) \log (t+1)$                                  & $- \log\rbr*{1-\exp(t^*)}$                                                                           & $-\log (t)-\log (t+1)$                       & $\log(1+t)$\\
            \bottomrule
        \end{tabular}
    }
\end{table}

An important tool of ours are the \emph{$f$-divergences} of information theory  \citep{ali1966general,csiszar1967information}. The $f$-divergence of $P$ from $Q$ is
\begin{align}
    \I_f(P,Q) \defas \int f\rbr{\diff PQ}\d Q,\label{eq:f_div_defn}
\end{align}
where it is assumed that  $f: \R\to\Rx$ is convex and lower semi-continuous, and $Q$ dominates $P$.\footnote{Common divergence measures such as Kullback--Liebler (KL) and total variation can easily be shown to be members of this family by picking $f$ accordingly \citep{reid2011information}. Several examples of these are listed in \autoref{tab:common_divergences}.}  Every $f$-divergence has a \emph{variational representation} \citep{reid2011information} via the Fenchel conjugate:
\begin{align}
    \I_f(P,Q)
    &= \sup_{u\in(\dom f^*)^{\cal X}}\rbr\Big{\E_Pu - \E_Q{f^*°u}},\label{eq:var_rep}
\end{align}
where the supremum is implicitly over all measurable real functions. 

In contrast to the abstract family $(\dom f^*)^{\cal X}$, machine learning models tend to be specified in terms of \emph{density ratios}, \emph{binary conditional distributions}, and \emph{binary classifiers}, these are respectively\footnote{While it might seem like there are certain inclusions here (for example $\prob\subseteq\radon\subseteq\disc$), these categories of functions really are distinct objects when thought of with respect to their corresponding binary classification decision rules (listed in \autoref{tab:decision_rules}).}
\begin{gather}
   \radon \defas \cbr{d:\cal X \to (0,\infty)}, 
   \mkern12mu \prob \defas \cbr{D:\cal X \to (0,1)}, 
   \mkern12mu \disc \defas \cbr{c:\cal X \to \R}.
\end{gather}
It is easy to see that these sets are all isomorphic with the commonly used connections
\begin{gather}\mathclap{
    \phi(D)\defas \frac D{1-D},
    \mkern12mu \sigma(c)\defas \frac1{1+\exp(-c)},
    \mkern12mu (\phi\circ\sigma) = \exp,
}\end{gather}
which are illustrated in \autoref{fig:iosmorphic}.\footnote{It is worth noting the selection of isomorphisms here is not unique and can be designed from principles involving loss functions \citep{rwCB}. 
}

It is a common result \citep{nguyen2010estimating,grover2017boosted,nowozin2016f} that the supremum in \eqref{eq:var_rep} is achieved for $f'\circ \diff P/Q$. It's convenient to define the \emph{reparameterised variational problem}:
\begin{gather}
    \maximise_u
    J(u) \defas \E_P f'\circ u - \E_Q f^*\circ f'\circ u
    \st
    u\in \cal F.
    \tag{V}\label{eq:var_prob}
\end{gather}
This reparameterisation can be justified in the case where $f$ is strictly convex since $f'$ a bijection.\footnote{What one may lose with this reparameterisation is convexity---that is, $J$ is not necessarily convex. }

\begin{figure}
    \begin{floatrow}
    \ffigbox{%
        \tikzsetnextfilename{ml_isomorphisms}
        \begin{tikzpicture}[every node/.style={font=\small}]
            \node (F) {$\intr(\dom f^*)^{\cal X}$};
            \node (R) [left = 0.8 cm of F] {$\radon$};
            \node (C) [left = 2cm of R] {$\disc$};
            \draw[->, sloped, midway, anchor=center, above] (R) to node {$f'\circ$} (F);
            \draw[->, sloped, midway, anchor=center, above, bend left=20] (C) to node (Exp) {$\exp\circ$} (R);
            \node (D) [below = 0.6cm of Exp] {$\prob$};
            \draw[->, sloped, midway, anchor=center, below, bend right=20] (C) to node {$\sigma\circ$} (D);
            \draw[->, sloped, midway, anchor=center, below, bend right=20] (D) to node {$\phi\circ$} (R);
        \end{tikzpicture}%
    }{%
        \caption{Isomorphisms for reparameterising \eqref{eq:var_prob}.\label{fig:iosmorphic}}
    }
    \capbtabbox{%
        \begin{tabular}[b]{cccc}\toprule
            Collection & $\radon$ & $\prob$ & $\disc$ \\
            Decision rule & $d\geq 1$ & $D\geq 1-D$ & $c\geq 0$ \\
            \bottomrule
        \end{tabular}
    }{%
        \caption{Classification decision rules.\label{tab:decision_rules}}%
    }
    \end{floatrow}
\end{figure}

\begin{example}[Neural Classifier]
    When $d\in\radon$ is modelled using a neural network classifier $D\in\disc$, $d$ can be efficiently calculated by substituting the softmax layer with
    \begin{gather}
      \exp c(x) =  \phi\rbr{\frac{\exp{c(x)}}{1-\exp{c(x)}} } = \phi\circ D(x) = d(x),
    \end{gather}
    where $c(x)$ is the neural network potential at $x\in\cal X$. This is just the arrow $\disc \to \radon$ in \autoref{fig:iosmorphic}.
\end{example}

 \begin{example}[$f$-GAN]\label{ex:gan}
    The GAN objective \citep{goodfellow2014generative} is an example of \eqref{eq:var_prob}:
    \begin{align}
        \sup_{D\in \prob}\rbr{\E_P \log( D ) + \E_Q \log(1-D)}
        &= \sup_{D\in \prob}\rbr{\E_P (f'\circ \phi)\circ D - \E_Q (f^*\circ f'\circ \phi)\circ D}
        \\&= \gan(P,Q),
    \end{align}
    where the function $f$ is defined in \autoref{tab:common_divergences} corresponding to the GAN $f$-divergence. In our derivation it's clear that \eqref{eq:var_prob} together with the isomorphisms in \autoref{fig:iosmorphic} give a simple, principled choice for the ``output activation function'', $g_f$, of \citet{nowozin2016f}.
\end{example}

\section{Boosted density estimation}\label{sec:boosted_density_estimation}


We fit distributions $Q_t$ for over the space $\cal X$ of the following form
\begin{gather}
   \forall{t\in\N} \d Q_t \defas \frac1{\int \prod_{i=1}^t d_i^{\alpha_i} \d Q_0 } \prod_{i=1}^t d_i^{\alpha_i}\d Q_0,\label{eq:multiplicative_density}
\end{gather}
where $Q_0$ is an initial reference distribution, $(\alpha_t)_{t\in\N}\subseteq[0,1]$ are the step sizes (for reasons that will be clear shortly), and $d_t:\cal X \to \R_+$ are some positive functions. The previous density can be expressed recursively as follows:
\begin{gather}
    \d \tilde Q_t = d_t^{\alpha_t}·\d \tilde Q_{t-1},\quad Q_t = \frac1{Z_t}\tilde Q_t,\quad Z_t\defas\int\d\tilde Q_t.\label{eq:rec_update}
\end{gather}
\begin{restatable}{proposition}{prop:recursive_normalisation} 
    The normalisation factors can be written recursively with $Z_t = Z_{t-1}·\E_{Q_{t-1}}d_t^{\alpha_t}$.
\end{restatable}
\begin{restatable}{proposition}{prop:qt_exponential_family} 
    Let $Q_t$ be defined via \eqref{eq:multiplicative_density} with a sequence of binary classifiers $c_1,\dots,c_t\in\disc $, where $c_i = \log d_i$ for $i\in [t]$. Then $Q_t$ is an exponential family distribution with natural parameter $\alpha\defas(\alpha_1,\dots,\alpha_t)$ and sufficient statistic $c(x)\defas(c_1(x),\dots,c_t(x))$.
\end{restatable}
The sufficient statistic of our distributions are classifiers that
would hence be learned, along with the appropriate fitting of natural
parameters. As explained in the proof, the representation may not be
minimal; however, without further constraints on $\alpha$, the
exponential family is regular \citep{bIA}. A similar interpretation of a neural network in terms of parameterising the sufficient statistics of a \emph{deformed exponential family} is given by \citet{nock2017f}.

In the remainder of this section, we show how to learn the classifiers from $c$
and fit the natural parameters $\alpha$ from (observed) data to ensure
a convergence $Q_t \rightarrow_t P$ that fits to the boosting framework.

\subsection{General convergence results of $Q_t$ to $P$}

The updates $d_t$ are chosen by taking the minimiser in \eqref{eq:var_prob}. To make explicit the dependence of $Q_t$ on $\alpha_t$ we will sometimes write $\d \tilde Q_t|_{\alpha_t} \defas d_t^{\alpha_t} \d \tilde Q_{t-1}$ and $\d Q_t|_{\alpha_t}\defas \frac1{Z_t}\d\tilde Q_t|_{\alpha_t}$. Since $Q_t$ is an exponential family (\autoref{prop:qt_exponential_family}), we measure the divergence between $P$ and $Q_t$ using the KL divergence (\autoref{tab:common_divergences}), which is the canonical
divergence of exponential families \citep{anMO}. Notice that we can write any solution to
\eqref{eq:var_prob} as $d_t = \diff P/{Q_{t-1}}·\epsilon_t$, where $\epsilon_t:\cal X \to \R_+$ is called the \emph{error term} due to the fact that that it is parameterised by the difference between the constrained solution to \eqref{eq:var_prob} and the global solution. A more detailed analysis of the quantity $\epsilon_t$ is presented in \autoref{sec:the_error_term}. 

\begin{restatable}{theorem}{thm:kl_bound}
    For any $\alpha_t\in[0,1]$, letting $Q_t,Q_{t-1}$ as in \eqref{eq:rec_update},  we have:
    \begin{gather}
        \begin{aligned}
            \forall{d_t \in \radon}
            \kl(P,Q_{t}|_{\alpha_t}) 
            &\leq (1-\alpha_t)\kl(P,Q_{t-1})  + \alpha_t\rbr{\log\E_P\epsilon_t - \E_P\log\epsilon_t}.
        \end{aligned}
        \label{eq:kl_bound}
    \end{gather}
    where $d_t = \diff P/{Q_{t-1}}·\epsilon_t$.
\end{restatable}
    We emphasize the fact that \autoref{thm:kl_bound} holds for any update $d_t$, but in fact for all possible functions $\cal X \to \R_+$, covering all ways --- and thus applying to all related algorithms --- for computing $Q_t$ as in \eqref{eq:multiplicative_density}.
\begin{remark} 
    \citet{grover2017boosted} assume a uniform error term, $\epsilon_t\equiv 1$. In this case \autoref{thm:kl_bound} yields geometric convergence
    \begin{gather}
           \forall{\alpha_t\in[0,1]} \kl(P,Q_t|_{\alpha_t}) \leq (1-\alpha_t) \kl(P, Q_{t-1}).\label{eq:simpe_kl_decrease}
    \end{gather}
    This result is significantly stronger than \citet[Theorem 2]{grover2017boosted}, who just show the non-increase of the KL divergence. If, in addition to achieving uniform error, we fix in this case $\alpha_t = 1$, then \eqref{eq:kl_bound} guarantees $Q_t|_{\alpha_t=1}$ is immediately equal to $P$.
\end{remark}

We can express the update \eqref{eq:multiplicative_density} and \eqref{eq:kl_bound} in a way that more closely resembles Frank--Wolfe update \eqref{eq:genupdatebden}. Since $\epsilon_t$ takes on positive values, we can identify it with a density ratio involving a (not necessarily normalised)
measure $\tilde R_t$, as follows
\begin{gather}
    \d \tilde R_t \defas \epsilon_t\cdot \d P \and  R_t \defas \frac1{\int \d\tilde R_t}·\tilde R_t.\label{eq:R_t_defn}
\end{gather}
Thus, in the case where $d_t$ is an inexact solution to the variational problem, the update \eqref{eq:rec_update} becomes
\begin{align}
    \mathclap{\d Q_t 
    \propto \rbr{\diff{ R_t}{Q_{t-1}} }^{\alpha_t} \d Q_{t-1}
    = (\d R_t)^{\alpha_t}·(\d  Q_{t-1})^{1-\alpha_t},} \label{eq_FWT}
\end{align}
which bears in disguise the Frank--Wolfe type update for iterate
$R_t$. We
insist on the fact that this iterate is unknown in general --- the parallel with Frank--Wolfe is therefore more
superficial than for \eqref{eq:genupdatebden}. Introducing $\tilde R_t$ allows us to lend some interpretation to \autoref{thm:kl_bound} in terms of the probability measure $R_t$.
\begin{restatable}{corollary}{cor:qt_rt}
For any $\alpha_t\in[0,1]$ and $\epsilon_t \in [0,+\infty)^{\cal X}$, letting $Q_t$ as in
    \eqref{eq:multiplicative_density} and $R_t$ from \eqref{eq:R_t_defn}. If $R_t$ satisfies
\begin{gather}
\kl\rbr{P,R_t} \leq \gamma \kl(P,Q_{t-1})\label{assumptionRT1}
\end{gather}
for $\gamma \in [0,1]$, then
\begin{gather}
\mathclap{\kl(P,Q_{t}|_{\alpha_t}) \leq
  (1-\alpha_t(1-\gamma))\kl(P,Q_{t-1}).} \label{bGamma1_}
\end{gather}
\end{restatable}
We obtain the same geometric convergence result as in
the best result of \citet[Corollary 2]{tgbssAB} for an update
$Q_t$ which is not a convex mixture, which, to our knowledge, is a new
result. \autoref{cor:qt_rt} is restricted to the KL divergence \textit{but} we do not need the technical domination
assumption of $Q_t$ with respect to $P$ that \citet[Corollary
2]{tgbssAB} requires. From the standpoint of weak versus strong
learning, \citet[Corollary 2]{tgbssAB} requires an condition similar to \eqref{assumptionRT1} --- 
iterate $R_t$ has to be close enough to the optimal iterate with
respect to $P$, while we assume that it
is close enough to $P$. We insist on the fact that such assumptions
are very strong. It is the objective of the following sections to
relax them to a setting compatible with boosting. 

\subsection{Convergence under weak learning assumptions}\label{ssec:convergence_under_weak_learning_assumptions}

In the previous section, we have established two basic convergence
results that compete with or beat the state of the art and rely on similar,
strong assumptions \citep{grover2017boosted,tgbssAB}. We now see
how to relax those assumptions in the context of boosting,
which requires polynomial time convergence under a learning
assumption as weak as possible. We now deliver such a result, starting
from the weak learning assumption at play.
Define the two \emph{expected  edges of $c_t$} \citep[cf.][]{nnOT}:
\begin{gather}
    \muq{t-1} \defas \frac1{\csup}\E_{Q_{t-1}}\sbr{-c_t}\and\mup \defas \frac1{\csup}\E_{P} [c_t].\label{eq:mu_defns}
\end{gather}
Classical boosting results rely on assumption on such edges for
different kinds of $c_t$ 
\citep{fsAD,sTS,ssIBj} and the implicit and weak assumption, that we also
make, that $\csup \defas \max_t \esssup |c_t| \ll \infty$, that is, the
classifiers have bounded confidence. By definition, $\muq{t-1}, \mup
\in [-1,1]$. The difference of sign of $c_t$ is due to the decision
rule for a binary classifier  (\autoref{tab:decision_rules}), whereby
$c_t(x)\geq 0$ reflects that $c_t$ classifies $x\in\cal X$ as
originating from $P$ rather than $Q_{t-1}$, and vice versa for
$-c_t(x)$.
\begin{assumption}[Weak Learning Assumption]
\begin{gather}
    \begin{aligned}
        \exists{\gammap,\gammaq \in(0,1]} \mup\geq \gammap,\quad \muq{t-1} \geq \gammaq.\hspace{4em}
    \end{aligned}\tag{WLA}\label{wla}
\end{gather}
\end{assumption}\noeqref{wla}

\begin{figure}
    \centering
    \makebox[\textwidth][c]{
        \pgfmathdeclarefunction{gauss}{2}{\pgfmathparse{1/(#2*sqrt(2*pi))*exp(-((x-#1)^2)/(2*#2^2))}}
        \pgfmathdeclarefunction{sigmoid}{1}{\pgfmathparse{1/(1+exp(-#1))}}
        \subfloat[\ref{wla} is satisfied since $\mup$ and $\muq{t-1}$ are postive.]{\label{fig:wla_illustraion1}
        \tikzsetnextfilename{wla_satisfied}
        \begin{tikzpicture}
            \small
            \begin{axis}[wla_diagrams]
                \def\bd{1.3}
                \addplot [dashed, thick, Red]  {gauss(-2,1)} node[left, pos=0] {$P$};
                \addplot [dashed, thick, Blue] {gauss(+2,1)}  node[right, pos=1] {$Q_{t-1}$};
                \addplot [thin, Blue] {0.5*(sigmoid(1*(x-\bd)) - 0.5)} node[right, pos=1] {$-\frac1{\csup}·c_t$};
                \addplot [thin, Red] {-0.5*(sigmoid(1*(x-\bd)) - 0.5)} node[left, pos=0] {$\frac1{\csup}·c_t$};
                \addplot [draw=Red, thick, name path=intp]   {-1.5*gauss(-2,1)*(sigmoid(1*(x-\bd)) - 0.5)} node[above, pos=0] {};
                \addplot [draw=Blue, thick, name path=intq]  {1.5*gauss(2,1)*(sigmoid(1*(x-\bd)) - 0.5)};
                \path    [name path=xaxis] (axis cs:-3.5,0) -- (axis cs:5,0);
                \addplot [LightRed] fill between[of=intp and xaxis];
                \addplot [LightBlue] fill between[of=intq and xaxis];
                \addplot +[dashed, black, thin, gray, mark=none] coordinates {(\bd, \pgfkeysvalueof{/pgfplots/ymin}) (\bd, \pgfkeysvalueof{/pgfplots/ymax})};

                \draw [Red,  -latex] (axis cs:-3.5, -0.1) node[left] {$\mup$}      to[out=0, in=245] (axis cs: -2,0.1);
                \draw [Blue, -latex] (axis cs: 5.0, -0.1) node[right]{$\muq{t-1}$} to[out=180,in=-45] (axis cs: 3,0.05);
            \end{axis}\end{tikzpicture}}
        \quad
        \subfloat[\ref{wla} is violated since $\muq{t-1}$ is negative.]{\label{fig:wla_illustraion2}
        \tikzsetnextfilename{wla_dissatisfied}
        \begin{tikzpicture}
            \small
            \begin{axis}[wla_diagrams]
                \def\bd{2.5}
                \addplot [dashed, thick, Red]  {gauss(-2,1)} node[left, pos=0] {$P$};
                \addplot [dashed, thick, Blue] {gauss(+2,1)}  node[right, pos=1] {$Q_{t-1}$};
                \addplot [thin, Blue] {0.5*(sigmoid(1*(x-\bd)) - 0.5)} node[right, pos=1] {$-\frac1{\csup}·c_t$};
                \addplot [thin, Red] {-0.5*(sigmoid(1*(x-\bd)) - 0.5)} node[left, pos=0] {$\frac1{\csup}·c_t$};
                \addplot [draw=Red, thick, name path=intp]   {-1.5*gauss(-2,1)*(sigmoid(1*(x-\bd)) - 0.5)} node[above, pos=0] {};
                \addplot [draw=Blue, thick, name path=intq]  {1.5*gauss(2,1)*(sigmoid(1*(x-\bd)) - 0.5)};
                \path    [name path=xaxis] (axis cs:-3.5,0) -- (axis cs:5,0);
                \addplot [LightRed] fill between[of=intp and xaxis];
                \addplot [LightBlue] fill between[of=intq and xaxis];
                \addplot +[dashed, black, thin, gray, mark=none] coordinates {(\bd, \pgfkeysvalueof{/pgfplots/ymin}) (\bd, \pgfkeysvalueof{/pgfplots/ymax})};

                \draw [Red,  -latex] (axis cs:-3.5, -0.1) node[left] {$\mup$}      to[out=0, in=245] (axis cs: -2,0.1);
                \draw [Blue, -latex] (axis cs: 5.0, -0.1) node[right]{$\muq{t-1}$} to[out=180,in=-45] (axis cs: 3.05,0.025);
            \end{axis}\end{tikzpicture}}
    }
    \caption{Illustration of \ref{wla} in one dimension with a classifier $c_t$ and its decision rule (indicated by the dashed grey line). The red (resp.\ blue) area is the area under the $c_t/\csup · \d P$  (resp.\ $-c_t/\csup · \d Q$) line, that is, $\mup$ (resp.\ $\muq{t-1}$).}
    \label{fig:wla_illustraion}
\end{figure}
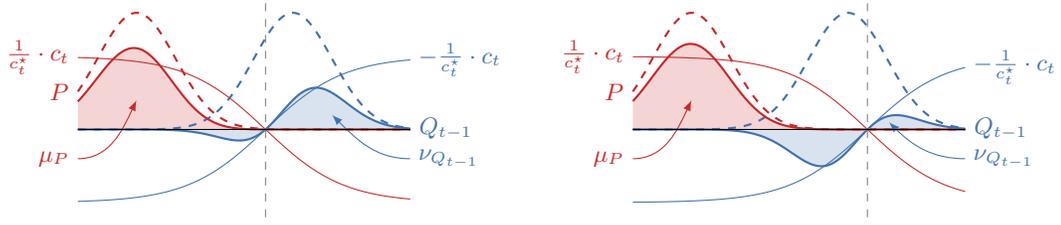


The Weak Learning Assumption is in effect a separation condition of
$P$ and $Q_{t-1}$. That is, the decision boundary associated with
$c_t$ correctly divides most of the mass of $P$ and most of the mass
of $Q$. This is illustrated in  \autoref{fig:wla_illustraion}. Note
that if $Q_{t-1}$ has converged to $P$, the weak learning assumption
cannot hold. This is reasonable since as $Q_{t-1}\to P$ it becomes
harder to build a classifier to tell them apart. We note that
classical boosting would rely on a single inequality for the weak
learning assumption (involving the two
edges) \citep{ssIBj} instead of two as in our \eqref{wla}. The
difference is however superficial as we can show that both assumptions
are equivalent (Lemma \ref{lemEQUIVWLA} in $\S$\ref{sec:allproofs}).
A boosting algorithm would ensure, for any given error $\varrho > 0$,
that there exists a number of iterations $T$ for which we do have
$\kl(P,Q_{T}) \leq \varrho$, where $T$ is required to be polynomial in
all relevant parameters, in particular $1/\gammap, 1/\gammaq, \csup,
\kl(P,Q_{0})$. Notice that we have to put $\kl(P,Q_{0})$ in the
complexity requirement since it can be arbitrarily large compared to
the other parameters.
\begin{restatable}{theorem}{rateWLA}
    Suppose \ref{wla} holds at each iteration. Then using $Q_t$ as in
    \eqref{eq:multiplicative_density} and $\alpha_t$ as in \autoref{ssec:convergence_under_weak_learning_assumptions}, we are
    guaranteed that $\kl(P,Q_{T}) \leq \varrho$ after a number of
    iterations $T$ satisfying: 
\begin{gather}
T \geq 2 \cdot \frac{\kl(P,Q_{0}) - \varrho}{\gammap\gammaq}.
\end{gather}
\end{restatable}
There is more to boosting: the
question naturally arises as to whether faster convergence is
possible. A simple observation allows to conclude that
it should require more than just \ref{wla}. Define
\begin{gather}
    \mu_{\epsilon_t} \defas \frac{1}{\csup}\cdot \E_P \log \epsilon_t,
    \label{defMuepsilont}
\end{gather}
the normalized expected log-density estimation error. Then we have $\mup = (1/\csup) \cdot
    \kl(P,Q_{t-1}) + \mu_{\epsilon_t}$,
so controlling $\mup$ does not give substantial leverage on
$\kl(P,Q_{t})$ because of the unknown $\mu_{\epsilon_t}$. We 
show that an additional weak assumption on $\mu_{\epsilon_t}$ is all
that is needed with \ref{wla}, to obtain convergence rates that compete with
\citet[Lemma 2]{tgbssAB} but using much weaker assumptions. We call
this assumption the \emph{Weak Dominance Assumption (WDA)}.

\begin{assumption}[Weak Dominance Assumption]
    \begin{gather}
        \exists{\gammae>0}\forall{t\geq 1}\mu_{\epsilon_t} \geq -\gammae\tag{WDA}\label{wda}
    \end{gather}
\end{assumption}
\noeqref{wda}
The assumption \ref{wda} takes its name from the observation that we have
\begin{gather}
    c_t = \log d_t = \log\rbr{\diff P{Q_{t-1}}·\epsilon_t}\quad\text{and}\quad|c_t| \leq\csup,
\end{gather}
so by ensuring that $\epsilon_t$ is going to be non-zero
$P$-almost everywhere, \ref{wda} states that nowhere in the support do we have
$Q_{t-1}.$ negligible against $P$. This also looks like a weak finite
form of absolute
continuity of $P$ with respect to $Q_{t-1}$, which is not
unreminiscent of the boundedness condition on the log-density ratio of
\citet[Theorem 1]{lbMD}. Provided \ref{wla} and
\ref{wda} hold, we are able to show geometric boosting convergence rates.
\begin{restatable}{theorem}{geomBOOST}
    Suppose \ref{wla} and \ref{wda} hold at each boosting iteration. Then we
    get after $T$ boosting iterations:
    \begin{align}
        \mathclap{\kl(P,Q_{T}) \leq 
        \rbr{1 - \frac{\min\{2, \gammaq /
            \csup\}\gammap}{2(1+\gammae)}}^T \cdot \kl(P,Q_{0}).}
    \end{align}
\end{restatable}
Hence, we get a boosting algorithm which
guarantees $\kl(P,Q_{T}) \leq \varrho$ after a number of
iterations $T$ which is now logarithmic in relevant parameters: $T\geq (1/\log K) \cdot \log(\kl(P,Q_{0})/\varrho)$,
where $K\defas 1/(1-(\min\cbr{2,  \gammaq / \csup}\gammap/(2(1+\gammae))))$.

\section{Experiments\protect\footnotemark}\label{sec:experiments}
\footnotetext{%
    The Julia-language code to run the subsequent experiments is made available at \href{https://github.com/ZacCranko/BoostedDensities.jl}{github.com/ZacCranko/BoostedDensities.jl}, complete with details about the implementation of kernel density estimation.
}

Let $t \in \cbr{0,\dots, T}$. Our experiments take place in $\cal X \defas \R^2$, where we use a simple neural network classifier $c_t\in\disc$, which we train (in the usual way) using cross entropy error by post composing it with the logistic sigmoid: $\sigma \circ c_t$. After training $c_t$ we transform it into to a density ratio using an exponential function: $d_t \defas \exp\mathop{\circ}c_t$ (cf.\ \autoref{sec:preliminaries}) which we use to update $Q_{t-1}$. 

In most experiments we train for $T> 1$ rounds therefore we need to
sample from $Q_{t-1}$.\footnote{It is easy to pick $Q_0$ to be
  convenient to sample from.} Our setting here is simple and so this
is easily accomplished using random walk Metropolis--Hastings. As
noted in the introduction, in more sophisticated domains it remains an
open question how to sample effectively from a density of the form
\eqref{eq:multiplicative_density}, in particular for a support having large dimensionality.

Since our classifiers $c_t$ are the outputs of a neural network they are unbounded, this violates the assumptions of \autoref{sec:boosted_density_estimation}, therefore in most cases we use the naive choice $\alpha_t \defas 1/2$.

\paragraph{Metrics}

At each $t\in\cbr{0,\dots, T}$ we estimate compute $\kl(P, Q_t)$, Negative Log-Likelihood (NLL) $\frac1{\E_{P}\log \d P}\E_{P}\log \d Q$, and accuracy $\E_P\mathord{}\iver*{c_t > \frac12}$. Note that we normalise NLL by its true value to make this quantity more interperable. The KL divergence is computed using numerical integration, and as such it can be quite tricky to ensure stability when running stochastically varying experiments, and becomes very hard to compute in dimensions higher than $n=2$. In these computationally difficult cases we use NLL, which is much more stable by comparison. We plot the mean and $95\%$ confidence intervals for these quantities. 

\subsection{Results}
Complete details about the experimental procedures including target data and network architectures are deferred to the supplementary material (\autoref{sec:experimental_setup}).

\subsubsection{Error and convergence}\label{ssec:error_and_convergence}
\begin{wrapfigure}{R}{.5\textwidth}
    \centering
        \tikzsetnextfilename{training_error_accuracy}
        \begin{tikzpicture}[inner sep=0.2em, outer sep=0.5\pgflinewidth]
            \begin{groupplot}[julia_convergence_plot, group style={group size= 1 by 2,  xlabels at=edge bottom, xticklabels at=edge bottom, vertical sep=1em}, xmin=0, xmax=5, xtick = {0,...,6}, width = 7cm]
            \nextgroupplot[ymin=0.4, ymax=0.9, height=3cm, ylabel style={align=center, at={(0,0.5)}, yshift=1.6em}, ylabel={Accuracy\\(0.5 if $Q_t\to P$)}, extra y ticks = {0.5}]
                \addplot[forget plot] graphics [xmin=0, xmax=7, ymin=0, ymax=1.5, includegraphics={trim=4mm 2.1mm 7mm 2.1mm, clip}, ] {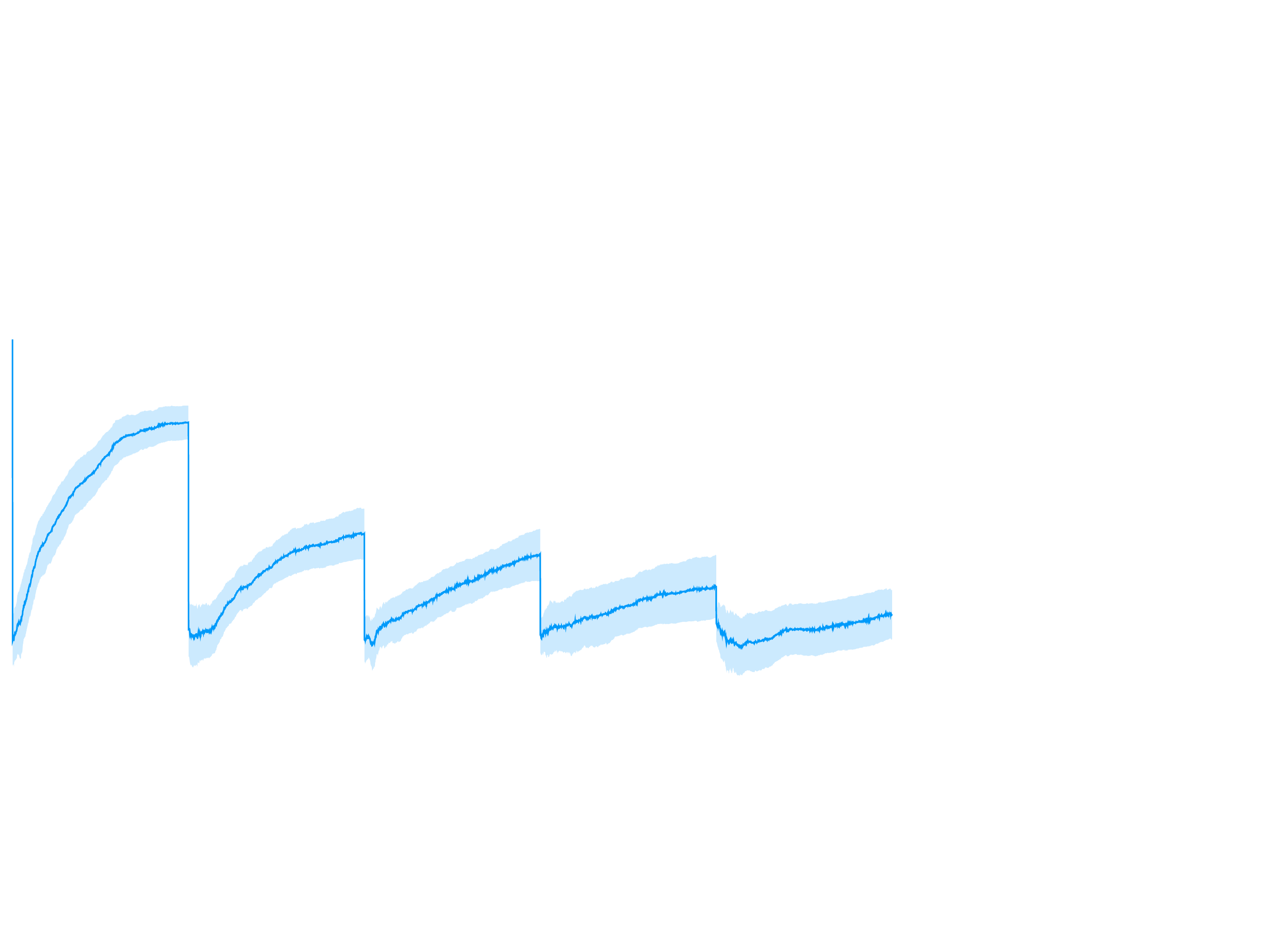};
            \nextgroupplot[ymin=0.0, ymax=4, ylabel style={align=center, at={(0,0.5)}, yshift=1.6em}, height=3cm, ylabel= {$\kl(P,Q_t)$\\(lower is better)}]
                \addplot[forget plot] graphics [xmin=0, xmax=7, ymin=0, ymax=4.0, includegraphics={trim=2.1mm 2.1mm 7mm 2.1mm, clip}]{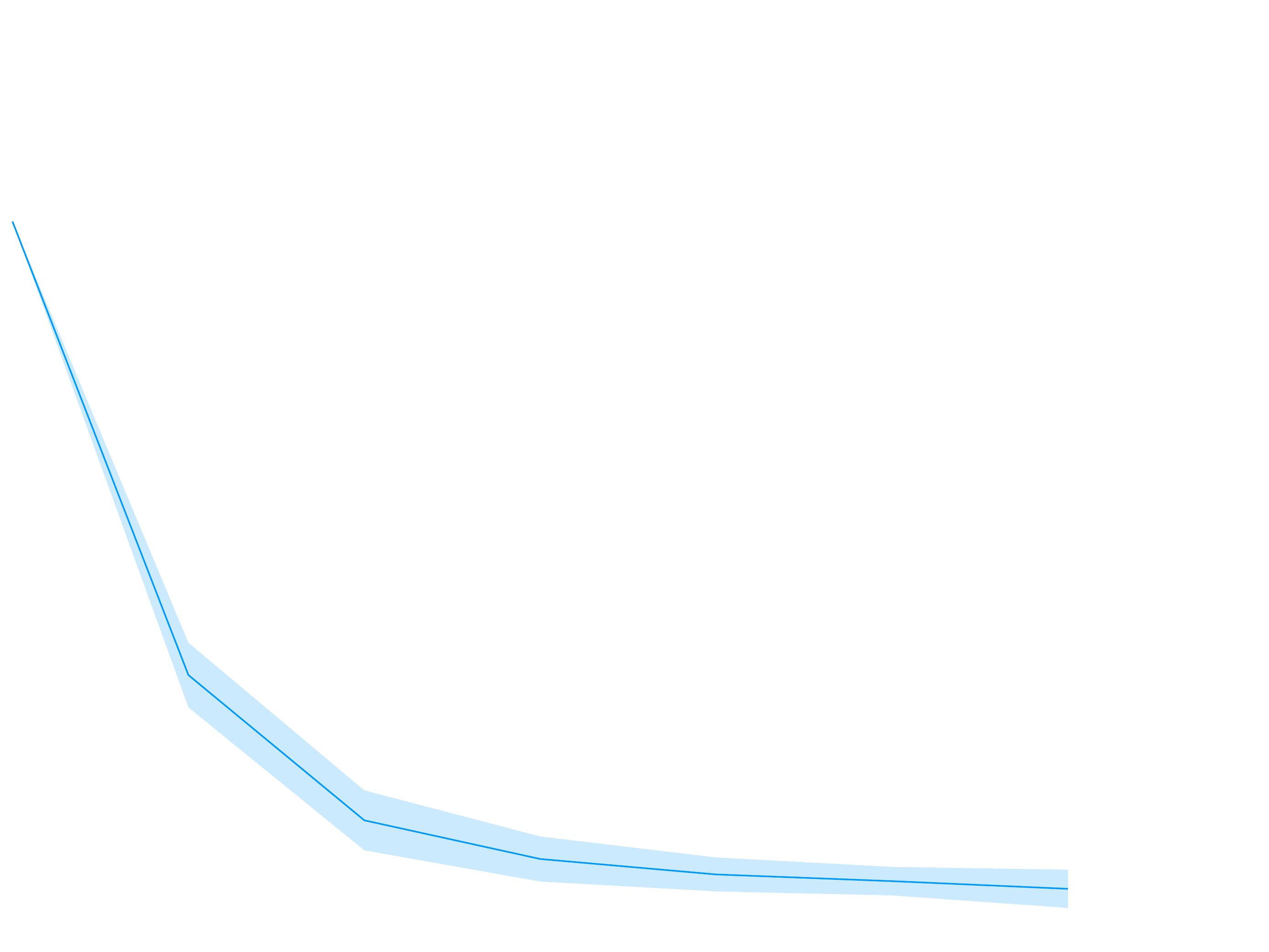};                
            \end{groupplot}
        \end{tikzpicture}
    \caption{As $\kl(P, Q_t)\to 0$ it becomes harder and harder to train a good classifier $c_t$. \label{fig:classifier_difficulties}}
\end{wrapfigure}

In order to minimise the divergence of $Q_t$ from $P$ we need to train a sufficiently good classifier $c_t$ such that we can build a good approximation to $\diff P/Q_{t-1}$. Naturally as  $Q_t \to P$ it should become harder and harder to tell the difference between a sample from $P$ and $Q_t$ with high probability.

This is exactly what we observe. In \autoref{fig:classifier_difficulties} we train a classifier with the same neural network topology as in \autoref{ssec:activation_functions}. The test accuracy over the course of training before each $t$ is plotted. As $\kl(P,Q_t)\to 0$ samples from $P$ and  $Q_t$ become harder and harder to tell apart and the best accuracy we can achieve over the course of training decreases, approaching $1/2$. Dually, the higher the training accuracy achieved by $c_{t}$, the greater the reduction from $\kl(P,Q_{t-1})$ to $\kl(P,Q_{t})$, thus the decreasing saw-tooth shape in \autoref{fig:classifier_difficulties} is characteristic of convergence.

\subsubsection{Activation functions}\label{ssec:activation_functions}
To look at the effect of the choice of activation function $a$ we train the same network topology, for a set of activation functions:  Numerical results trained to fit a ring of Gaussians are plotted in \autoref{fig:activation_comparison_kl},  contour plots of some of the resulting densities are presented \autoref{fig:activation_contour_plots}. All  activation functions except for $\mathrm{Softplus}$ performed about the same by the end of six round, with $\mathrm{ReLU}$ and $\mathrm{SELU}$ being the marginal winners. It is also interesting to note the narrow error ribbons on $\tanh$ compared to the other functions indicating more consistent training.

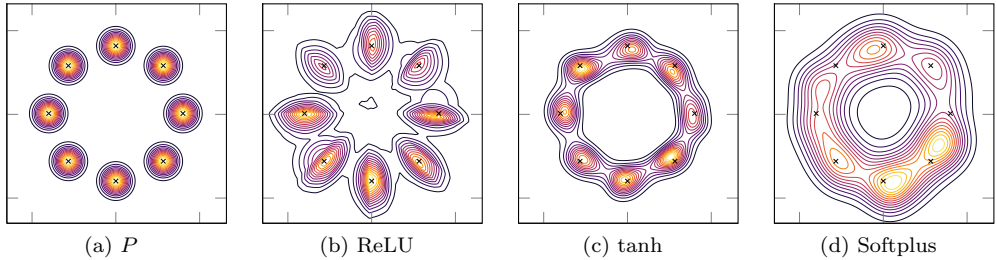
\begin{figure}[t]
    \begin{center}
        \subfloat[$P$]{\label{fig:activation_target}\tikzsetnextfilename{activation_target}
            \begin{tikzpicture}\begin{axis}[julia_contour_plot, xmin=-6.5, xmax=6.6, ymin=-6.5, ymax=6.6]
                \addplot graphics [xmin=-6.5, xmax=6.6, ymin=-6.5, ymax=6.6, includegraphics={trim=16 16 20 16, clip}] 
                {plots/activation_comparison/activation_comparison-control-Plots_contour-lims_-6_5_to_6_6};
            \end{axis}\end{tikzpicture}}\quad
        \subfloat[$\mathrm{ReLU}$]{\label{fig:activation_relu}\tikzsetnextfilename{activation_relu}
            \begin{tikzpicture}\begin{axis}[julia_contour_plot, xmin=-6.5, xmax=6.6, ymin=-6.5, ymax=6.6]
                \addplot graphics [xmin=-6.5, xmax=6.6, ymin=-6.5, ymax=6.6, includegraphics={trim=16 16 20 16, clip}] 
                {plots/activation_comparison/activation_comparison-NNlib_relu-Plots_contour-lims_-6_5_to_6_6};
            \end{axis}\end{tikzpicture}}\quad
        \subfloat[$\tanh$]{\label{fig:activation_tanh}\tikzsetnextfilename{activation_tanh}
            \begin{tikzpicture}\begin{axis}[julia_contour_plot, xmin=-6.5, xmax=6.6, ymin=-6.5, ymax=6.6]
                \addplot graphics [xmin=-6.5, xmax=6.6, ymin=-6.5, ymax=6.6, includegraphics={trim=16 16 20 16, clip}] 
                {plots/activation_comparison/activation_comparison-tanh-Plots_contour-lims_-6_5_to_6_6};
            \end{axis}\end{tikzpicture}}\quad
        \subfloat[$\mathrm{Softplus}$]{\label{fig:activation_softplus}\tikzsetnextfilename{activation_softplus}
            \begin{tikzpicture}\begin{axis}[julia_contour_plot, xmin=-6.5, xmax=6.6, ymin=-6.5, ymax=6.6]
                \addplot graphics [xmin=-6.5, xmax=6.6, ymin=-6.5, ymax=6.6, includegraphics={trim=16 16 20 16, clip}] 
                {plots/activation_comparison/activation_comparison-NNlib_softplus-Plots_contour-lims_-6_5_to_6_6};
            \end{axis}\end{tikzpicture}}
    \end{center}
    \caption{The effect of different activation functions, modelling a ring of Gaussians. The ``petals'' in the $\mathrm{ReLU}$ condition are likely due to the linear hyperplane sections the final network layer being shaped by the final exponential layer.}
    \label{fig:activation_contour_plots}
\end{figure}

\subsubsection{Network topology}\label{ssec:architecture_comparison_nll}
To compare the effect of the choice of network architecture we fix activation function and try a variety of combinations of network architecture, varying both the depth and the number nodes per layer. For this experiment the target distribution $P$ is a mixture of 8 Gaussians that are randomly positioned at the beginning of each run of training. Let $m\times n$ denote a fully conencted neural network $c_t$ with $m$ hidden layers and $n$ nodes per layer. After each hidden layer we apply the $\mathrm{SELU}$ activation function. 

Numerical results are plotted in \autoref{fig:architecture_comparison_nll}. Interestingly doubling the nodes per layer has little benefit, showing only modererate advantage. By comparison, increasing the network depth allows us to achieve over a 70\% reduction in the minimal divergence we are able to achieve. 

\begin{figure}[t]
    \begin{center}
        \subfloat[Activation function experiment]{\label{fig:activation_comparison_kl}
            \tikzsetnextfilename{activation_comparison_kl}
            \begin{tikzpicture}[inner sep=0.2em, outer sep=0.5\pgflinewidth, trim axis left, trim axis right]
                \begin{axis}[julia_convergence_plot, xmin=0.0, xmax=6.0, ymin=0.0, ymax=6.5, ylabel style={align=center},
                ylabel= {$\kl(P,Q_t)$\\(lower is better)}, xtick = {0,...,6}]
                \addplot[forget plot] graphics [xmin=0, xmax=7, ymin=0, ymax=7, includegraphics={trim=2.1mm 2.1mm 7mm 2.1mm, clip}]
                {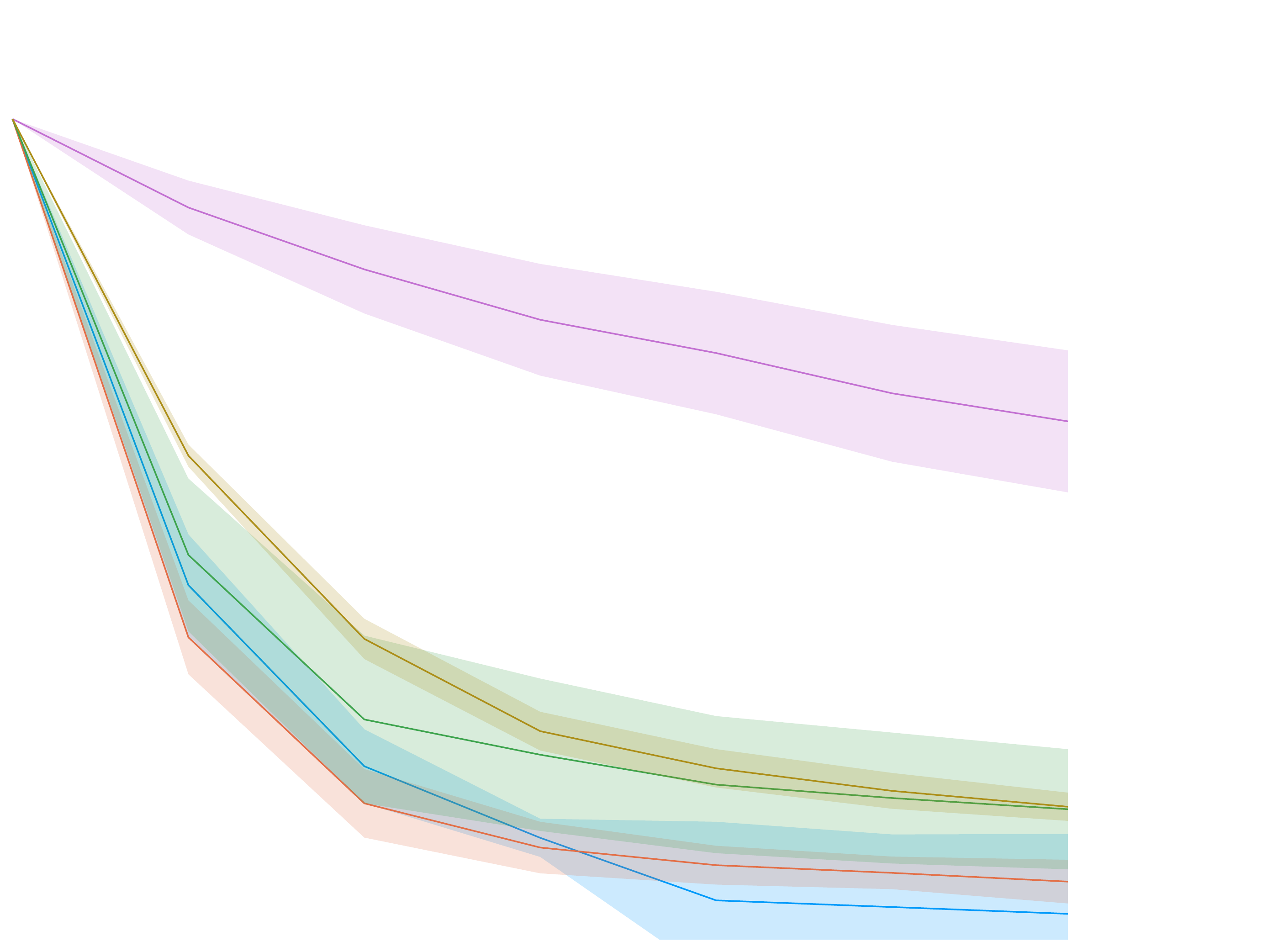};
                \addlegendimage{no markers, plots_01}
                \addlegendentry{$\mathrm{ReLU}$ }
                \addlegendimage{no markers, plots_02}
                \addlegendentry{$\mathrm{SELU}$ }
                \addlegendimage{no markers, plots_03}
                \addlegendentry{$\mathrm{Softplus}$ }
                \addlegendimage{no markers, plots_04}
                \addlegendentry{$\mathrm{Sigmoid}$ }
                \addlegendimage{no markers, plots_05}
                \addlegendentry{$\tanh$}
            \end{axis}\end{tikzpicture}}
        \qquad\qquad
        \subfloat[Network topology experiment]{\label{fig:architecture_comparison_nll}
            \tikzsetnextfilename{network_architecture_comparison_kl}
            \begin{tikzpicture}[inner sep=0.2em, outer sep=0.5\pgflinewidth, trim axis left, trim axis right]
                \begin{axis}[julia_convergence_plot, xmin=0.0, xmax=6.0, ymin=0.0, ymax=0.8, ylabel style={align=center},
                ylabel= {$\kl(P,Q_t)$\\(lower is better)}, xtick = {0,...,6}]
                \addplot[forget plot] graphics [xmin=0, xmax=6, ymin=0, ymax=0.8, includegraphics={trim=2.1mm 2.1mm 7mm 2.1mm, clip}]
                {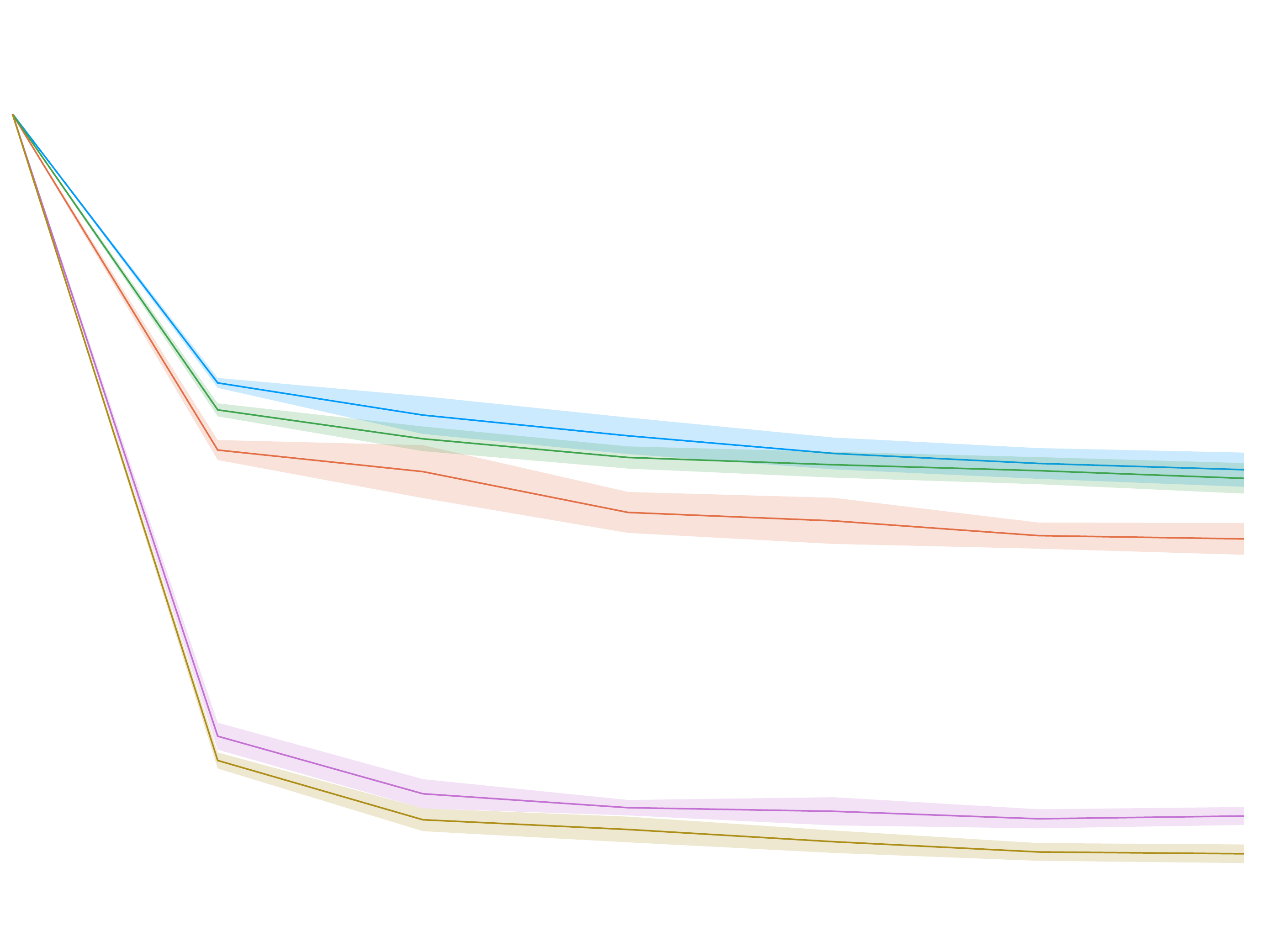};
                \addlegendimage{no markers, plots_01}
                \addlegendentry{$\mathrm{SELU}:1\times 5$ }
                \addlegendimage{no markers, plots_02}
                \addlegendentry{$\mathrm{SELU}:2\times 5$ }
                \addlegendimage{no markers, plots_03}
                \addlegendentry{$\mathrm{SELU}:1\times 10$ }
                \addlegendimage{no markers, plots_04}
                \addlegendentry{$\mathrm{SELU}:2\times 10$ }
                \addlegendimage{no markers, plots_05}
                \addlegendentry{$\mathrm{SELU}:2\times 20$}
            \end{axis}\end{tikzpicture}}\\
    \end{center}
    \caption{KL divergence for a variety of activation funtions and architectures over six iterations of boosting.}
\end{figure}

\subsubsection{Convergence across dimensions}\label{ssec:convergence_across_dimensions}

For this experiment we vary the dimension $n \in \cbr{2,4,6}$ of the
space $\cal X = \R^n$ using a neural classifier $c_t$ that is trained
without regard for overfitting and look at the convergence of NLL
(\autoref{fig:dimensionality}). After we achieve the optimal NLL of 1,
we observe that NLL becomes quite variable as we begin to
overfit. Secondly overfitting the likelihood becomes harder as we
increase the dimensonality, taking roughly two times the number of
iterations to pass NLL $=1$ in the $n=4$ condition as the $n=2$
condition. We conjecture that not overfitting is a matter of early
stopping boosting, in a similar way as it was proven for the
consistency of boosting algorithms \citep{btAI}.

\begin{wrapfigure}{R}{.5\textwidth}
    \tikzsetnextfilename{dimensonality_experiment}
    \begin{tikzpicture}[inner sep=0.2em, outer sep=0.5\pgflinewidth, trim axis right]
        \begin{axis}[julia_convergence_plot, xmin=0, xmax=10, ymin=0.0, ymax=2.5, ylabel style={align=center, yshift=-1em},
        ylabel= {Negative Log-Likelihood\\(closer to 1 is better)}, xtick = {0,...,10}, extra y ticks = {1.0}]
        \addplot[forget plot] graphics [xmin=-1, xmax=11, ymin=-1, ymax=3, includegraphics={trim=3mm 2.1mm 7mm 2.1mm, clip}]{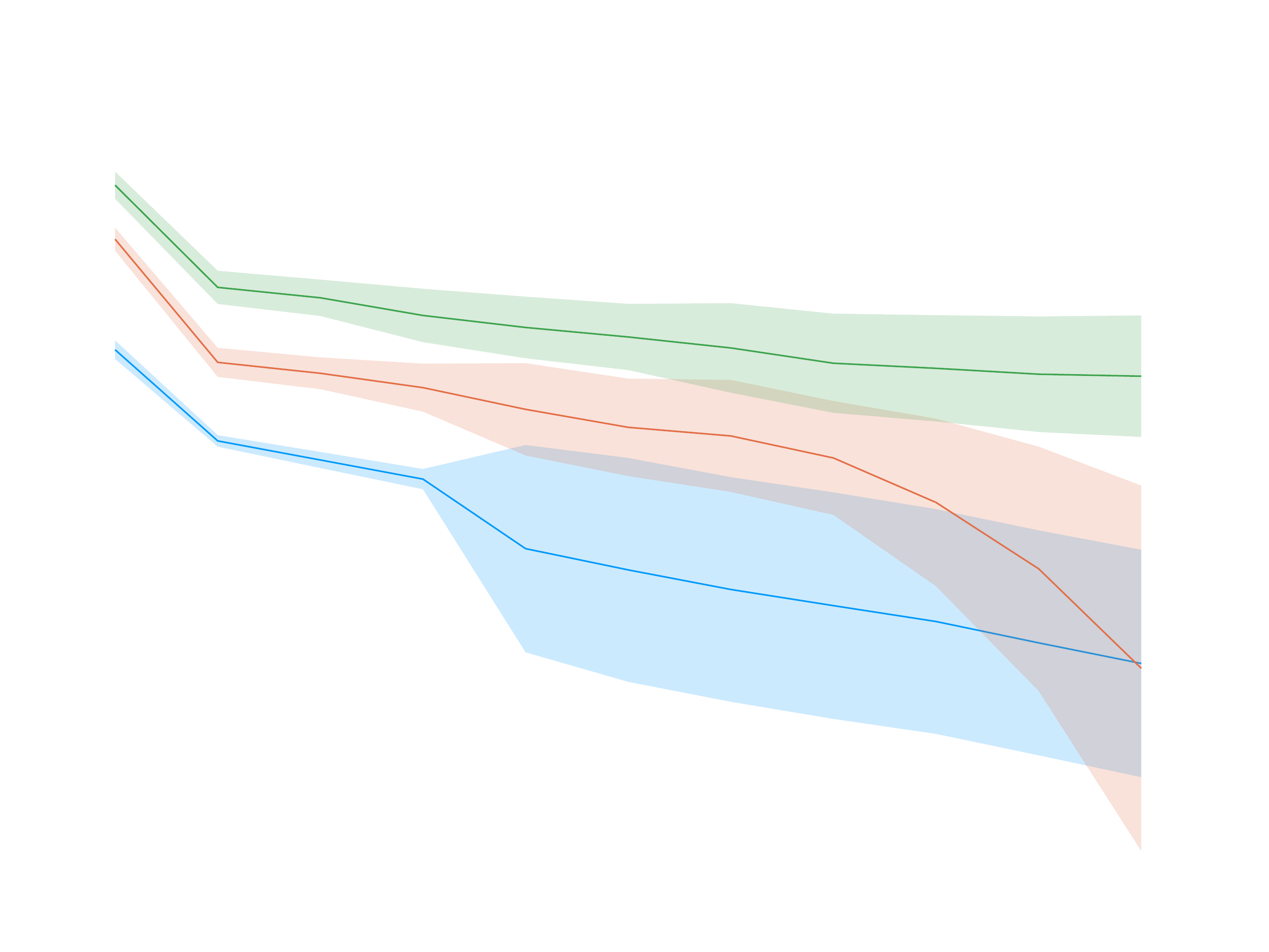};
        \addlegendimage{no markers, plots_01}
        \addlegendentry{$2$}
        \addlegendimage{no markers, plots_02}
        \addlegendentry{$4$}
        \addlegendimage{no markers, plots_03}
        \addlegendentry{$6$}
    \end{axis}\end{tikzpicture}
    \caption{Convergence in more dimensions.\label{fig:dimensionality}}
\end{wrapfigure}

\subsubsection{Comparison with kernel density estimation}\label{ssec:comparison_with_kde}

In this experiment we compare our boosted densities with Kernel Density Estimation (KDE). For this experiment we train a deep neural network with three hidden layers.
The step size $\alpha$ is selected to minimise NLL by evaluating the training set at 10 equally spaced points over $[0,1]$. We compare the resultant density after $T=2$ rounds with a variety of kernel density estimators, with bandwidth selected via the Scott/Silverman rule.\footnote{The Scott and Silverman rules yield identical bandwidth selection criteria in the two-dimensional case.}

\begin{figure}[t]
    \newsavebox{\kdeplot}
    \savebox{\kdeplot}{\tikzsetnextfilename{kde_nll_violin_plot}
    \begin{tikzpicture}[inner sep=0.2em, outer sep=0.5\pgflinewidth, trim axis left, trim axis right]
        \begin{axis}[julia_violin_plot, width = 9cm, 
        xmin=0.0, xmax=9.0, ymin=0.5, ymax=1.75,
        ylabel style={align=center}, 
        ylabel= {Negative Log-Likelihood\\(closer to 1 is better)},
        xtick={0.5, 1.5, 2.5, 3.5, 4.5, 5.5, 6.5, 7.5, 8.5},
        yticklabel style={font=\tiny},
        xticklabel style={rotate=35, font=\tiny, anchor= north east},
        xticklabels={
            $Q_0 $               ,
            \textsc{Cosine}      ,
            $Q_2 $               ,
            \textsc{Triangular}  ,
            \textsc{Exponential} ,
            $Q_1 $               ,
            \textsc{Gaussian}    ,
            \textsc{Epanechnikov},
            \textsc{Tophat}      
        }]
            \addplot graphics [xmin=0.0, xmax=9.0, ymin=0.0, ymax=2.0, includegraphics={trim=3.0mm 3.0mm 5.0mm 4.0mm}] {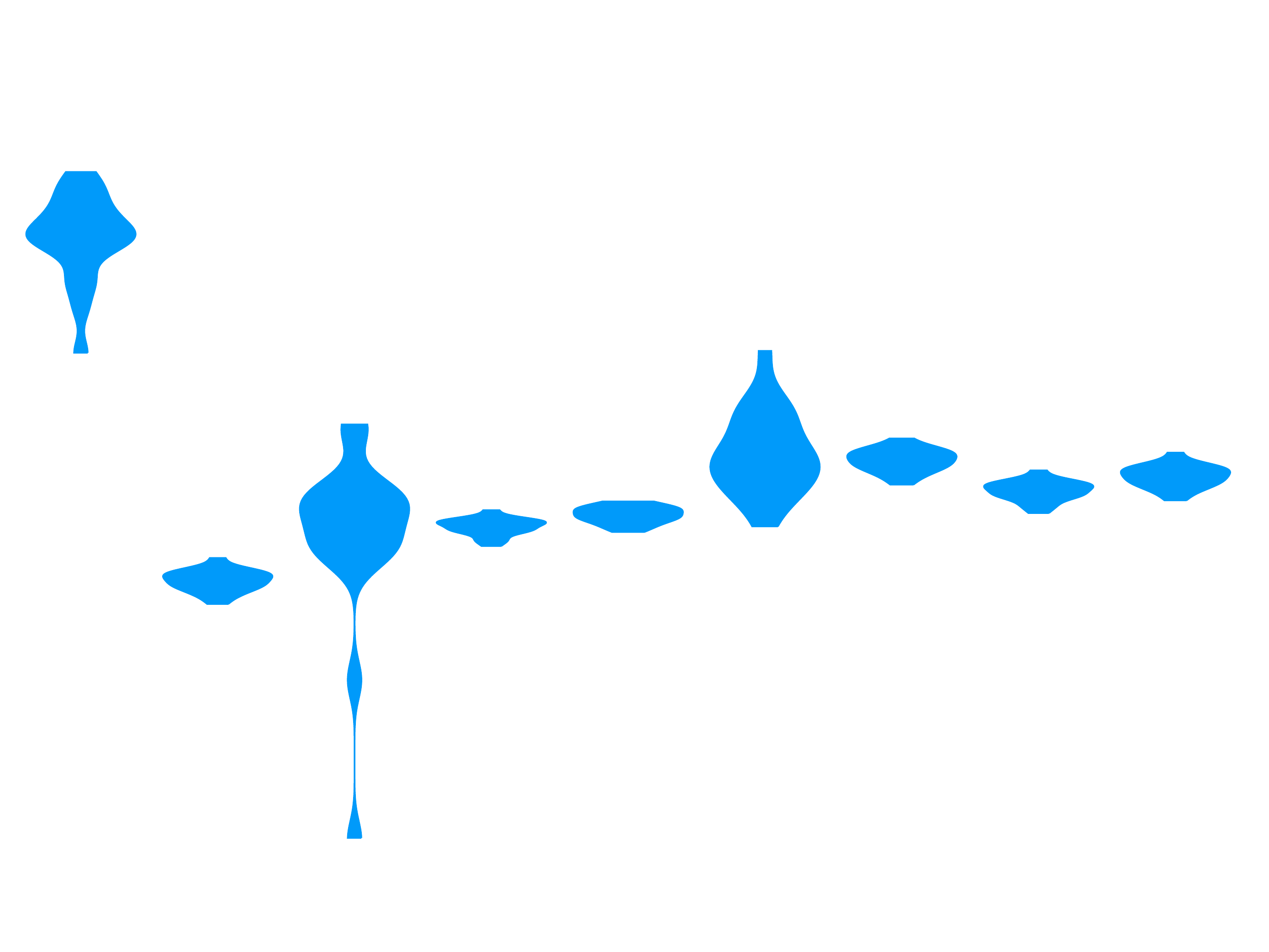};
            \addplot[dashed, black, samples=100, domain=0:9, gray] {1};
        \end{axis}
    \end{tikzpicture}}
    \centering
    \makebox[\textwidth][c]{
    \subfloat[\label{fig:kde_comparison_experiment_statistics}]{%
        \raisebox{\dimexpr.5\ht\kdeplot-.5\height}{%
        \footnotesize
        \begin{tabular}[b]{l c}\toprule
            Condition &\begin{tabular}[x]{@{}c@{}} Mean NLL {\scriptsize$\pm$95\% CI}\\ \end{tabular}\\\midrule
            $Q_0 $                   & 1.5131 {\scriptsize$\pm 0.0459$} \\
            \textsc{Cosine}          & 0.7734 {\scriptsize$\pm 0.0112$} \\
            $Q_2 $                   & 0.8685 {\scriptsize$\pm 0.0946$} \\
            \textsc{Triangular}      & 0.8898 {\scriptsize$\pm 0.0089$} \\
            \textsc{Exponential}     & 0.9154 {\scriptsize$\pm 0.0088$} \\
            $Q_1 $                   & 1.0492 {\scriptsize$\pm 0.0437$} \\
            \textsc{Gaussian}        & 1.0333 {\scriptsize$\pm 0.0118$} \\
            \textsc{Epanechnikov}    & 0.9675 {\scriptsize$\pm 0.0105$} \\
            \textsc{Tophat}          & 0.9983 {\scriptsize$\pm 0.0117$} \\\bottomrule
        \end{tabular}
        }
    }\hspace{4em}\subfloat[\label{fig:kde_comparison_experiment_violin}]{%
        \centering
        \usebox{\kdeplot}
    }}
    \caption{KDE comparison results. The conditions are in decreasing order with respect to the absolute difference of mean NLL and 1. \label{fig:kde_comparison_experiment}}%
\end{figure}

Results from this experiment are diplayed in
\autoref{fig:kde_comparison_experiment}. On average $Q_1$ fits the
target distribution $P$ better than all but the most efficient
kernels, and at $Q_2$ we begin overfitting, which aligns with the
observations made in $\S$\ref{ssec:convergence_across_dimensions}. We
note that his performance is with a model with around 200 parameters,
while the kernel estimators each have 2000 --- \textit{i.e.} we
achieve KDE's performances with models whose size is the \textit{tenth} of
KDE's. Also, in this experiment $\alpha_t$ is selected to minimise NLL, however it is not hard to imagine that a different selection criteria for $\alpha_t$ would yield better properties with respect to overfitting.

\subsubsection{Comparison with \protect{\citet{tgbssAB}}}\label{ssec:comparison_with_adagan}

To compare the performance of our model (here called \textsc{Discrim}) with \textsc{AdaGAN} we replicate their Gaussian mixture toy experiment,\footnote{This is the eperiment {\texttt gaussian\_gmm.py} at \href{https://github.com/tolstikhin/adagan}{github.com/tolstikhin/adagan}} fitting a randomly located eight component isotropic Gaussian mixture where each component has constant variance. These are sampled using the code provided by \citet{tgbssAB}. 


We compute the coverage metric\footnote{The coverage metric $C_{\kappa}$ can be a bit misleading since any density $Q$ that covers $P$ will yield high $C_{\kappa}(P,Q)$, nomatter how spread out it is. This is the case at $t=0$ when we initially fit $Q_0$. A high coverage metric, however, is sufficient to claim that a model $Q$ has not ignored any of the mass of $P$ when combined with another metric such as NLL. That is, a high $C_{\kappa}$ is a necessary condition for mode-capure.} of \citet{tgbssAB}: $C_{\kappa}(P,Q) \defas P\rbr{\lev_{>\beta}Q}$, where  $Q\rbr{\lev_{>\beta}\d Q} = \kappa,$ and $\kappa\in[0,1]$. That is, we first find $\beta$ to determine a set where most of the mass of $Q$ lies, $\lev_{>\beta} \d Q$, then look at how much of the mass of $P$ resides there.

Results from teh experiment are plotted in \autoref{fig:adagan_comparison}. Both \textsc{Discrim} and \textsc{AdaGAN} converge closely to the true NLL, and then we observe the same characteristic overfitting in previous experiments after iteration 4 (\autoref{fig:adagan_comparison_nll}). It is also interesting that this also reveals itself in a degredation of the coverage metric \autoref{fig:adagan_comparison_coverage}. 

Notably \textsc{AdaGAN} converges tightly, with NLL centered around
its mean, while \textsc{Discrim} begins to vary wildly. However the
\text{AdaGAN} procedure includes a step size that decreases with $1/t$
--- thereby preventing to some extent overfitting with $t$ ---, whereas \textsc{Discrim} uses a constant step size 1/2. Suggesting that a similarly decreasing procedure for $\alpha_t$ may have desirable properties.

\begin{figure}[t]
    \begin{center}
        \subfloat[]{\label{fig:adagan_comparison_nll}
            \tikzsetnextfilename{adagan_comparison_nll}
            \begin{tikzpicture}[inner sep=0.2em, outer sep=0.5\pgflinewidth, trim axis left, trim axis right]
                \begin{axis}[julia_convergence_plot, xmin=0, xmax=8, ymin=-5, ymax=15, ylabel style={align=center, yshift=-1em}, ytick = {-5, 1, 5, 10, 15},
                ylabel= {Negative Log-Likelihood\\(closer to 1 is better)}, xtick = {0,...,10}, extra y ticks = {1.0}, minor y tick num = 4,, height=5cm]
                \addplot[forget plot] graphics [xmin=-1, xmax=10, ymin=-20, ymax=20, includegraphics={trim=3mm 2.1mm 7mm 2.1mm, clip}]
                {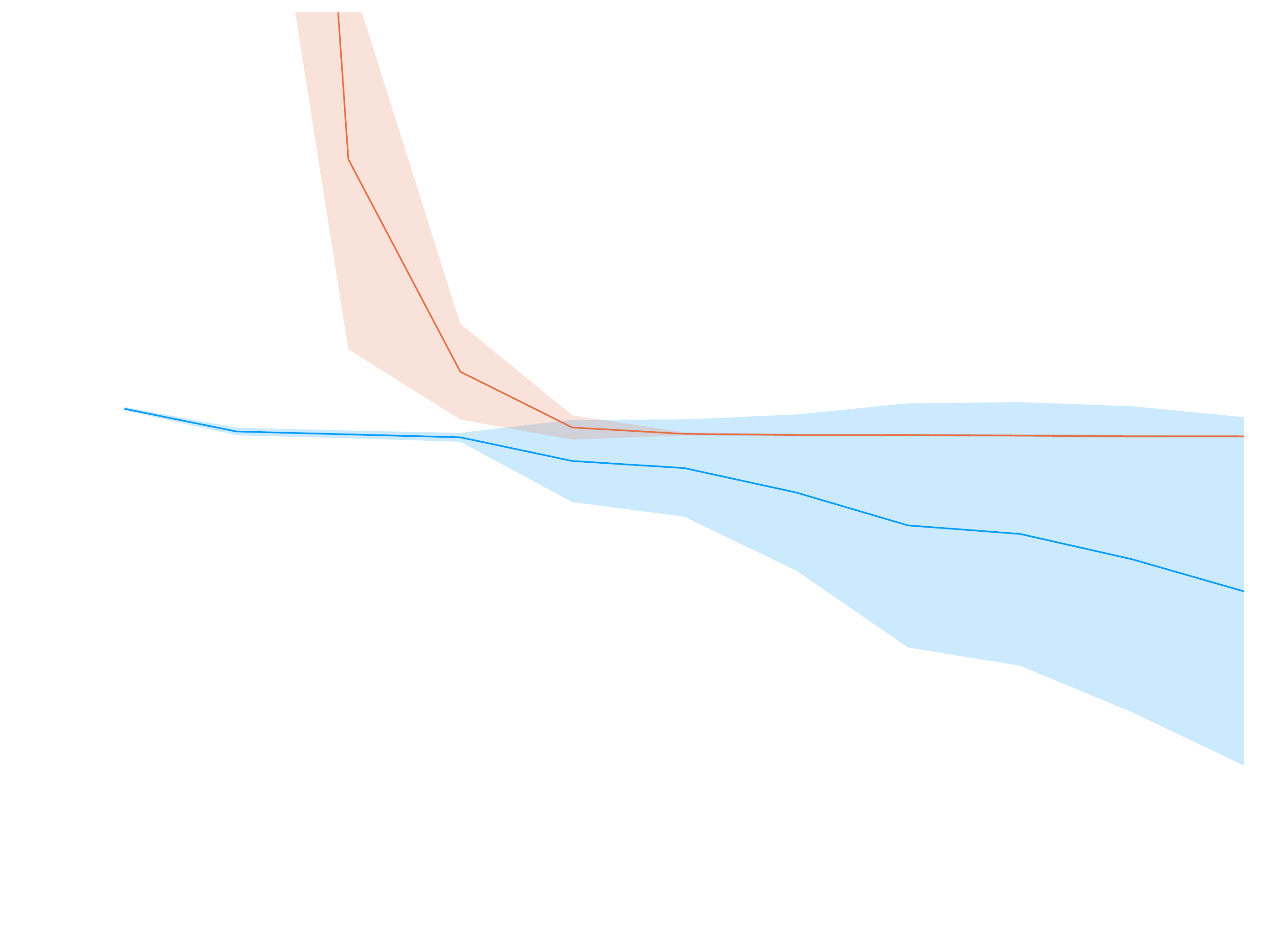};
                \addlegendimage{no markers, plots_01}
                \addlegendentry{\textsc{Discrim}}
                \addlegendimage{no markers, plots_02}
                \addlegendentry{\textsc{AdaGAN}}
            \end{axis}\end{tikzpicture}}
        \qquad\qquad
        \subfloat[]{\label{fig:adagan_comparison_coverage}
            \tikzsetnextfilename{adagan_comparison_coverage}
            \begin{tikzpicture}[inner sep=0.2em, outer sep=0.5\pgflinewidth, trim axis left, trim axis right]
                \begin{axis}[julia_convergence_plot, xmin=0, xmax=8, ymin=0.6, ymax=1.0,ylabel style={align=center, yshift=-1em},
                ylabel= {$C_{95\%}$\\(higher is better)}, xtick = {0,...,10}, legend style={yshift=-1.2cm}, extra y ticks = {1.0}, height=5cm]
                \addplot[forget plot] graphics [xmin=-1, xmax=10, ymin=0.1, ymax=1.5, includegraphics={trim=3mm 2.1mm 7mm 2.1mm, clip}]
                {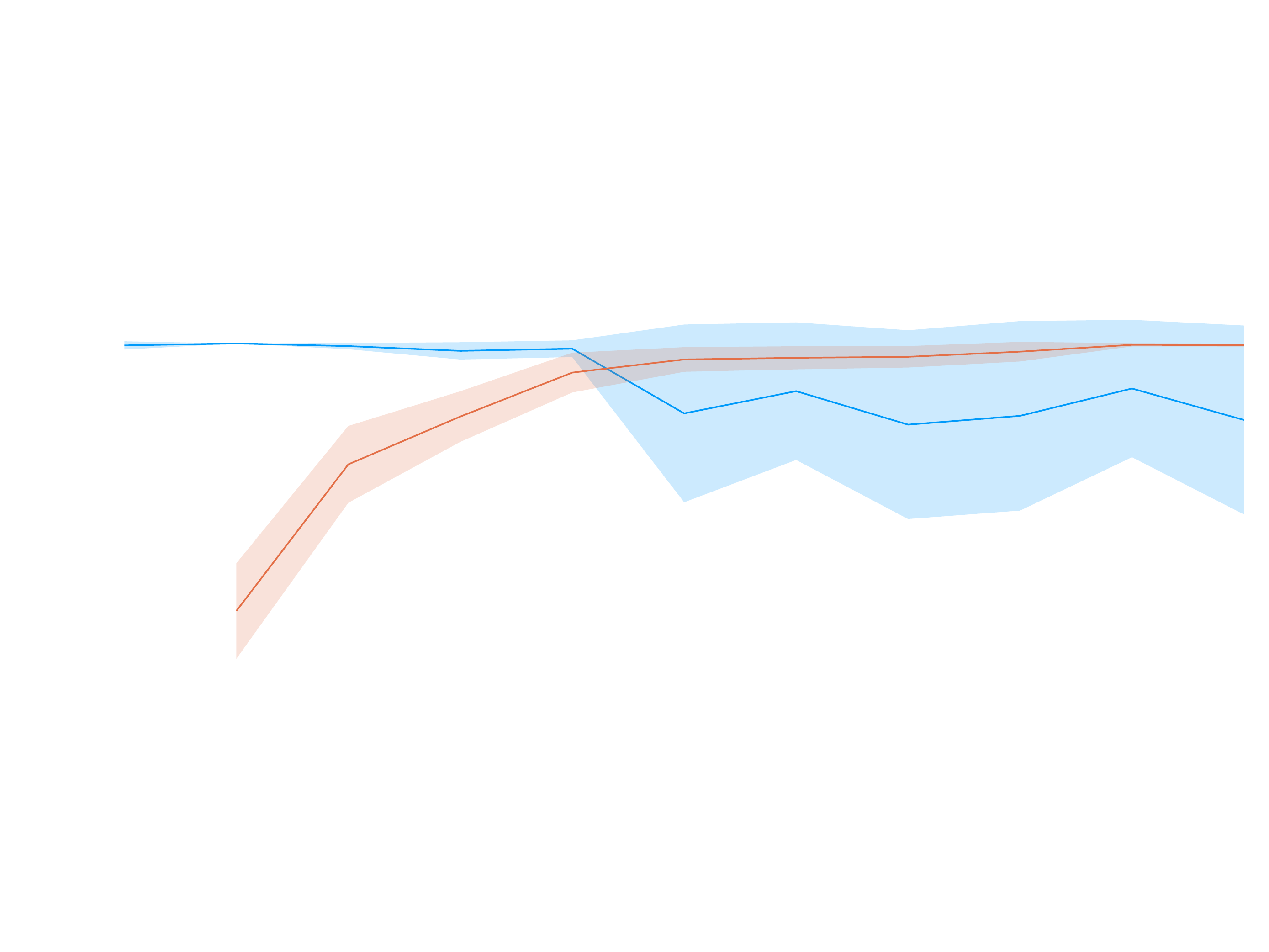};
                \addlegendimage{no markers, plots_01}
                \addlegendentry{\textsc{Discrim}}
                \addlegendimage{no markers, plots_02}
                \addlegendentry{\textsc{AdaGAN}}
            \end{axis}\end{tikzpicture}}
    \end{center}
    \caption{Comparing the performance of \textsc{Discrim} and \textsc{AdaGAN}. \label{fig:adagan_comparison}}
\end{figure}

\subsection{Summary}

We summarize here some key experimental observations:

\begin{itemize}
    \item both the activation functions and network topology have a large effect on the ease of training and the quality of the learned density $Q_T$ with deeper networks with fewer nodes per layer yielding the best results (\autoref{ssec:activation_functions}, \autoref{ssec:architecture_comparison_nll}).

    \item When the networks $c_t$ are trained long enough we observe overfitting in the resulting densities $Q_T$ and instability in the training procedure after the point of overfitting (\autoref{ssec:convergence_across_dimensions} \autoref{ssec:comparison_with_kde}, \autoref{ssec:comparison_with_adagan}), indicating that a procedure to take $\alpha_t\to 0$ should be optimal.

    \item We were able to match the performance of kernel density estimation with a naive procedure to select $\alpha_t$. With a better selection procedure we may very well be able to do much better, but this is beyond the scope of our perliminary investigation here (\autoref{ssec:comparison_with_kde}).
    \item We were able to at least match the performance of \textsc{AdaGAN} with respect to density estimation (\autoref{ssec:comparison_with_adagan}). 
\end{itemize}
Finally, while we have used KDE as a point of comparison of algorithm, there is no reason why the two techniques could not be combined. Since KDE is a closed form mixture distribution that's quite easy sampled, there is no reason why one couldn't build some kind of kernel density distribution and use this for $Q_0$ which one could refine with a neural network.

\section{Conclusion}\label{sec:conclusion}

The idea of learning a density in an iterative, ``boosting''-type
combination of proposal densities has recently met with significant attention. Typically, aproaches use a 
subroutine oracle to fit the coefficients in the
combination of densities. In all cases that have significant convergence
rates, such an oracle is required to satisfy very strong
constraints.

In this paper, we have shown that all it takes to learn a density
iteratively in a boosting fashion is a weak learner in the original sense of the Probably Approximately
Correct learning model of \citet{kTO}, leading to comparable or better convergence
bounds than previous approaches at comparatively very reduced price in
assumptions. We derive this result through a series of related
contributions, including (i) a finer characterization of the solution to
the $f$-GAN problem and (ii) a full characterization of the distribution we
learn in exponential families. 

Experimentally, our approach shows very promising results for an early
capture of modes, and significantly outperforms AdaGAN
during the early boosting iterations using a comparatively very small
architecture. Our experiments leave however open the challenge to obtain a
black box sampler for domains with moderate to large dimension. We
conjecture that the full characterization that we get of the distribution learned might be of significant help to tackle this challenge.

\newpage
\bibliography{bibliography}

\newpage
\makeatletter
\vbox{%
    \hsize\textwidth
    \linewidth\hsize
    \vskip 0.1in
    \@toptitlebar
    \centering
    {\Large\bf Supplementary Material for \\ \papertitle\par}
    \@bottomtitlebar
}
\vskip 0.3in \@minus 0.1in
\makeatother

\appendix

\section{Epilogue}\label{sec:epilogue}

Results in the area of iterative approaches to density estimation,
including boosting approaches which are iterative by nature, can be
characterized according to three features: how the \textit{convergence} is
characterized, what kind of \textit{assumptions} it does rely upon and finally,
whether it is of direct relevance to current \textit{empirical settings} for machine learning.

Regarding convergence, there are three kinds of formal results that are traditionally proven. Some are
convergence without rates \citep{grover2017boosted,dpsPG}, and others give rates
that are negligible with regard to recent results (including ours) \citep{rsBD}. In the third and final category are
explicit convergence rates. Some of the related approaches have
an explicit intractable objective and they rather optimize a tractable
surrogate bound. This is the case for variational inference, where the
surrogate is the evidence lower-bound \citep{gwfbdBV,kblssFS,lkgrBV,mfaVB}. Because of the explicit gap to the
intractable optimum, we do not mean to compare such approaches to
ours, but can summarize most of the formal results in those papers as showing sublinear
convergence, that is, of the form $I(P,Q_T) \leq \inf_Q I(P,Q) +
J/T$ for $J>0$ parameter dependant and $I$ defining a suitable divergence. 

In the rest of the related approaches, it quite remarkable that all of them exploit the same Frank--Wolfe-type update \eqref{eq:genupdatebden} \citep{lbMD,neDE,tgbssAB,zSG} --- even
when the connection to Frank--Wolfe is explicit in few of them 
\citep{lkgrBV}. Until recently \citep{tgbssAB}, all these other approaches essentially displayed sublinear convergence rates \citep{lbMD,neDE,zSG}. This can be
compared to our rates from \autoref{thBoostALPHA0} and 
\autoref{rateWLA}. We compare favorably with them from three standpoints. First,
all these algorithms integrate calls to
an oracle/subroutine that needs to solve a nested optimization
problem \textit{for its optimum} --- the contraint put on our oracle, the weak learner, appears
much weaker. Second, all these algorithms integrate parameters whose
computation would require the full knowledge of
distributions \citep{neDE,zSG} or their parameterized space
\citep{lbMD}. It is unclear how approximations would impact
convergence \citep{mfaVB}. In our case, \autoref{rateWLA} just operates on
estimated parameters, straightforward to compute. Third and last, previous works make more stringent structural
assumptions restricting the form of the optimum \citep{lbMD,neDE,zSG},
while we just assume that $\csup$ is bounded, which puts a constraint
--- easily enforceable ---
on the proposals of the weak learner and not on the optimum.

To drill down further in the \textbf{assumptions} required, the very few previous approaches that manage to beat sublinear
convergence to reach geometric convergence --- that we reach in
\protect{\autoref{geomBOOST}} --- require very
strong assumptions, such as the constraint that iterates are close
enough to the
optimum sought \citep[Corollaries 1, 2]{tgbssAB}. In fact, in this
latter work, the
parameterization of the weight $\alpha$ in
\eqref{eq:genupdatebden} chosen for their experiments \textit{implicitly imposes} the
convergence of iterates to this optimum \citep[\S4]{tgbssAB}. In our
case, we have shown that equivalent convergence rates can be
obtained without boosting (\autoref{cor:qt_rt}) but with an assumption
which is used in \citep[Corollary
1, Eq. 10]{tgbssAB}, and is thus very strong. Even when this is not our main result,
\autoref{cor:qt_rt} is new and interesting in the light of \citep{tgbssAB}'s
results because (i) it does not make use of their
convex mixture model and (ii) we do not have the
additional technical requirement that $P(\d Q_{t-1} / \d P = 0) < \alpha_t$,
that is, roughly the mass where $\d Q_{t-1} = 0$ is bounded by
the \textit{leveraging} coefficient. Our main result on geometric
convergence shows that such convergence is within reach with
a much weaker
assumptions than \citep[Corollary
1, Eq. 10]{tgbssAB}, in fact as weak as the weak learning
assumption. To get our result, we need an additional assumption on the lower-boundedness of the log-errors
$\epsilon_t$ a. e.
via the \ref{wda}, but this is still very weak considering that we
fit an exponential family and in interesting applications like
image processing, domain $\cal X$ is closed so unless $\d P$ is allowed
to peak arbitrarily, we essentially get \ref{wda} for reasonable $\gammae$.

Now, why is the assessment of all assumptions important in the light
of experimental settings? Because it brings them to
a trial by fire, as to whether results survive to experimental machine
learning, with available information which is in general a
partial estimated snapshot of the theory. It should be clear at this
point that, with the
\textit{sole} exception of a \textit{subset of} variational approaches ---
which, again, settle for an explicitly tractable surrogate of the objective ---, \textit{all previous
approaches} would fail at this test,
\citep{grover2017boosted,gwfbdBV,lbMD,lkgrBV,neDE,tgbssAB,zSG}. They
would all fail
essentially because in practice, we obviously would not have access to
$P$ to test assumptions nor carry out fine-grained optimisation
involving $P$. To our
knowledge, our result in \autoref{sec:boosting_with_estimates} is the
first attempt to provide an algorithm fully executable on current
experimental learning settings and whose convergence relies on assumptions that
would also easily be testable or enforceable empirically. 

We must insist however on the fact that all previous approaches that
investigate variational inference or GANs get a black box sampler
which may be hard to train \textit{but} it always easy to sample from,
in particular
 high
dimensions \citep{gwfbdBV,kblssFS,lkgrBV,mfaVB,tgbssAB} --- this is
clearly where the bottleneck of our theory lies currently, even when our
experiments display that efficient sampling is available up to
moderate dimensions. We conjecture that it is much further scalable. 

\section{The error term}\label{sec:the_error_term}
Recall the reparameterised variational problem from \autoref{sec:preliminaries}
\begin{gather}
    \minimise_u
    J(u) \defas \E_Q f^*\circ f'\circ u - \E_P f'\circ u
    \st
    u\in \cal F.
    \tag{\ref{eq:var_prob}}\label{eq:var_prob_app}
\end{gather}

The solution to \eqref{eq:var_prob_app} easily follows when $\cal F$ is a large enough set of measurable functions \citep{nowozin2016f,nguyen2010estimating,grover2017boosted}. However when $\cal F$ is a more constrained class, a stronger result is necessary. Assume $\cal F$ is a subset of the normed space, $(\scr F, \abs{\marg})$. Let $\scr F^*$ be its continuous dual. The Fréchet normal cone (also called prenormal cone) of $\cal F \subseteq \scr F$ at $u\in\scr F$ is
 \begin{gather}
    \Nc_{\cal F}(u) \defas 
    \cbr{ u^* \in\scr F^* :\,\limsup_{\mathclap{\cal F \supseteq (v)\to u}}\frac{\inp{u^*, v-u}}{\abs{v-u}}\leq 0}.
\end{gather}
When $\cal F$ is convex, $\Nc_{\cal F}(u)$ is the ordinary normal cone.

\begin{theorem}\label{thm:inexact_sol}
    Assume $f: \R_+ \to \R_+$ is strictly convex and twice differentiable, and $\scr F$ is a normed space of functions $\cal X \to \intr(\dom f)$. Let $\cal F\subseteq\scr F $ and  $\bar u\in\argmin_{u\in \cal F}J(u)$. If $J$ is finite on a neighbourhood of $\bar u$, then
    \begin{gather}
        \bar u \in \diff{P}{Q} - \Nc_{\cal F}(\bar u).
    \end{gather}
    If, in addition, $\cal F$ is convex with $\diff P/Q \in \intr \cal F$, then $\bar u = \diff P/Q$.
\end{theorem}

\begin{proof}
    Because $f$ is twice differentiable on $\intr(\dom f)$, and $J$ is finite on a neighbourhood of $\bar u$, $J$ is Fréchet differentiable at $\bar u$ with
    \begin{align}
        J'(\bar u) &=  ((f^*)'\circ f'\circ \bar u) · (f''\circ \bar u) · \d Q - (f''\circ \bar u) ·\d P.
        \\&=  \bar u ·(f'' ° \bar u)·\d Q - (f''°\bar u)·\d P,
    \end{align}
    where $(f^*)' = (f')^{-1}$ since $f$ is strictly convex.
    By hypothesis $J$ attains its minimum on $\cal F$ at $\bar u$, thus Fermat's rule \citep[Theorem 2.97, p.~170]{penot2012calculus} yields
    \begin{align}
        0 \in J'(\bar u) + \Nc_{\cal F}(\bar u)
          &\iff 0 \in \bar u ·(f''\circ \bar u) ·\d Q  - (f''\circ \bar u)·\d P +  \Nc_{\cal F}(\bar u)
          \\&\iff 0 \in \bar u  - \frac{\d P}{\d Q} +  \frac1{(f''°\bar u)·\d Q}·\Nc_{\cal F}(\bar u)
          \\&\iff  \bar u \in \frac{\d P}{\d Q} -    \Nc_{\cal F}(\bar u),
     \end{align}
     where the final biconditional follows since $\Nc_{\cal F}(\bar
     u)$ is a cone.

    Now, suppose $\diff{P}/{Q} \in \intr\cal F$ with $\cal F$ convex. Then the Fr\'echet cone becomes usual normal cone \citep[Ex.~6, p.~174]{penot2012calculus},
    \begin{gather}
        \Nc_{\cal F}(\bar u) \defas \cbr{ u^* \in\scr F^* : \forall{v \in \cal F} \inp{u^*, v-\bar u}\leq 0}.
    \end{gather}
    It's immediate from the definition that $\Nc_{\cal F}$ always contains $0$. We use a contradiction to show that $\Nc_{\cal F}(\bar u)\subseteq\cbr{0}$. Take $z^*≠0 \in\Nc_{\cal F}(\bar u)$. Let $\cal F_{\bar u}\defas \cal F -\bar u$. First note that $\bar u \in \intr \cal F$ implies $0\in\intr \cal F_{\bar u}$. Thus there is a closed symmetric neighbourhood $U$ with $0\in U\subseteq\intr \cal F_{\bar u}$. The Hahn--Banach strong separation theorem \citep[Theorem~1.79, p.~55]{penot2012calculus} guarantees the existence of a vector $u\in U$ such that 
    \begin{gather}
        \inp{z^*,u} > 0 \iff \exists{v\in\cal F} \inp{z^*,v - \bar u} > 0, 
    \end{gather}
    contradicting the assumption $z^*\in\Nc_{\cal F}$. Thus $\Nc_{\cal F}(\diff P/Q) =\cbr{0}$.
\end{proof}

The set $\Nc_{\cal F}(\bar u)$ can be thought of as containing a direction $v$ that pulls $\diff{P}/{Q}$ to the constrained minimiser $\bar u$. This is illustrated in \autoref{fig:inexact_sol}. 

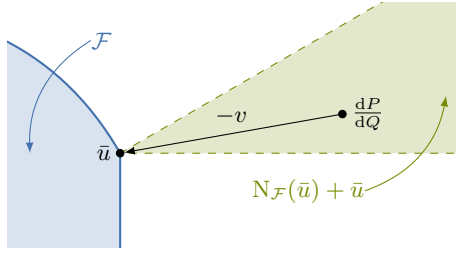
\begin{figure}
    \begin{center}
        \tikzsetnextfilename{inexact_solution_cone}
        \begin{tikzpicture}[every node/.style={font=\small}]
            \clip(-1.5,-1.25) rectangle (4.5,2);
            \def\a{30}
            \def\b{\a -30}
            \draw[thick, name path = f, Blue, fill=LightBlue] (-4,2) to [out=0, in=\a+90, looseness=1] (0,0) coordinate (u) to [out=\b+270, in=90] (0,-4) to (-4,-4) --cycle;
            \path[name path = lower] (-2,4) to (-2,-2) to (4,-2);
            \draw[dashed, Green, fill=LightGreen] (\a:6) to (u) to (\b:6);
            \draw (10:3) coordinate (dpdq); 
            \draw (dpdq) node[draw, circle, fill=black, inner sep=1pt] {} node[right] {$\diff P Q$};
            \draw (u)    node[draw, circle, fill=black, inner sep=1pt] {} node[left]  {$\bar u$};
            \draw (2.5,-0.5)  node [Green, inner sep = 0pt] (nc) {$\Nc_{\cal F}(\bar u)+ \bar u$};
            \draw (-0.25,1.5) node [Blue,  inner sep = 0pt] (nf) {$\cal F$};
            \draw[-latex, Green] (nc.east) to[bend right] (10:4.4);
            \draw[-latex, Blue]  (nf.west) to[bend right] (-1.2,0);
            \draw[-latex, shorten >=1.5pt] (dpdq) -- (u) node[pos = 0.5, above] {$-v$};
        \end{tikzpicture}
    \end{center}
    \caption{Illustration of \autoref{thm:inexact_sol} wherein there exists $v\in\Nc_{\cal F}(\bar u)$ which pulls the unconstrained minimiser, $\diff P/Q$, onto the constrained minimiser, $\bar u$.\label{fig:inexact_sol}}
\end{figure}

\autoref{thm:inexact_sol} also gives us give a more explicit characterisation of the error term in \autoref{sec:boosted_density_estimation} since 
\begin{gather}
    \exists{v_t \in \Nc_{\cal F}(d_t)} d_t = \diff P{Q_t} ·\epsilon_t  = \diff P{Q_t} - v_t,
    \iff
    \epsilon_t = 1 - 
    \diff{Q_{t-1}}P·v_t.\label{opt_epsilon}
\end{gather}

\section{Boosting with estimates}\label{sec:boosting_with_estimates}

In practice, we do not have access to $P$ and we rather sample from
$Q$. We thus assume the possibility to sample\footnote{We could also
  assume the availability of training samples, in particular for $P$
  as is usually carried out.} $P$ and $Q_.$ to compute all needed
estimates of $\mup$ and $\muq{.}$. 
So let us assume that the weak learner has access to a sampler of $P$
and a sampler of $Q_.$, ``\sampler''. \sampler~takes as input a
distribution and a natural $m$; it samples from the distribution and
returns an i.i.d. sample of size $m$. It does so separately for $P$
and $Q$, with separate sizes $m_P$ and $m_Q$ for the
respective samples. The full \autoref{alg:density_estpractEST} is very
similar to \autoref{alg:density_estpract} if we except the fact that
the weak learner also returns an estimate for $\muq{t-1}$. To analyze
\autoref{alg:density_estpractEST} requires however more than just the
boosting material developed so far, since nothing guarantees that
estimates of $\mup$ and $\muq{.}$ meeting \ref{wla} would imply
$\mup$ and $\muq{.}$ meeting \ref{wla} as well. We therefore
replace \ref{wla} by one which relies on \textit{estimates}
computed over large enough samples. We call it the \textit{Empirical}
Weak Learner Assumption (EWLA). It involves two additional parameters, $\kappasup =
\mucsup(1-\mucsup)/2$ ($>0$), where $\mucsup$,
which depends only on $\csup$, is
defined in \eqref{defMUCSUPMAIN}, and some $0<\delta\leq 1$.
\begin{assumption}[Empirical Weak Learner Assumption]
    There exists $\gammap,
    \gammaq \in (0,1]$ such that the following holds: at each iteration
    $t = 1, 2, ..., T$, \wl~estimates $\hatmup$ and $\hatmuq{t-1}$ respectively from
    \begin{gather}
        \sampler(P, m_P)\and\sampler(Q_{t-1}, m_Q)
    \end{gather}
    with $m_P, m_Q$ satisfying
    \begin{gather}
     m_P \geq \frac{1}{(\kappasup \gammap)^2} \log \frac{4T}{\delta} \and 
     m_Q \geq \frac{1}{(\kappasup \gammaq)^2} \log \frac{4T}{\delta},
     \tag{$\mathrm{EWLA}_{\delta,T}$}\qquad\qquad\label{ewla}
    \end{gather}\noeqref{ewla}
    and returns, along with $\hatmuq{t-1}$, $c_t$ satisfying
    $\hatmup\geq \gammap$ and $\hatmuq{t-1} \geq \gammaq$.
\end{assumption}
The weak learner thus also take as input $\delta$ and $T$, as
displayed in \autoref{alg:density_estpractEST}. We emphasize the fact that 
\ref{ewla} is just assuming the ability for \wl~to have i.i.d. samples
from $P$ and $Q$ and get a classifier $c_t$ that \textit{empirically}
satisfies \ref{wla}. Since we still focus on the decrease of
$\kl(P,Q_{.})$, one might expect this to weaken our results, which is
indeed the case, \textit{but} we can show that \textit{only} constants are
\textit{slightly} affected, thereby not changing significantly
convergence rates. We provide in one theorem the reframing of both \autoref{thBoostALPHA0} and \autoref{thBoostALPHA2}. In the same way as we did for \autoref{thBoostALPHA2}, whenever we are in the clamped regime for
$\alpha_t$, we let $\hat{\delta}_{t-1}\geq 0$ be defined from
$\hatmuq{t-1} = (1+\hat{\delta}_{t-1}) \mucsup$.
\begin{restatement}{theorem}{thBoostALPHA0PLUS}
    Suppose \ref{ewla} holds. Then with probability of at least $1 -
    \delta$, 
    \begin{align}
        \forall{t = 1, 2, ..., T}\kl(P,Q_{t})  \leq \kl(P,Q_{t-1})  - \Delta_t,
    \end{align}
    where 
    \begin{gather}
        \Delta_t \defas \begin{cases}
            \frac{\hatmup}{16} \log \rbr{\frac{1+\hatmuq{t-1}}{1-\hatmuq{t-1}}} & \text{in the non-clamped regime,}\\
            \frac{\hatmup \csup}{2} + \mucsup^2 \cdot \rbr{\frac{1}{4} + \frac{\hat{\delta}_{t-1} }{1-\mucsup^2}} & \text{otherwise.}
        \end{cases}
    \end{gather}
\end{restatement}

\begin{figure}
    \begin{minipage}[t]{0.45\textwidth}
        \vspace{0pt}
        \begin{algorithm}[H]
            \caption{AdaBoDE}
            \label{alg:density_estpract}
            \begin{algorithmic}
                \STATE {\bfseries Input:} distributions $P, Q_0$, natural $T$;
                \FOR{$t=1$ {\bfseries to} $T$}
                    \STATE $c_t \leftarrow \wl(P, Q_{t-1})$;
                    \STATE Pick $\muq{t-1}$ as in \eqref{eq:mu_defns} and $\alpha_t$ as
                    \STATE $\alpha_t \gets \min\cbr{1, \frac{1}{2\csup} \log\rbr{\frac{1+\muq{t-1}}{1-\muq{t-1}}}}$;
                    \STATE Pick $Q_{t}$ as in \eqref{eq:multiplicative_density} with $d_t \defas \exp
                    \circ c_t$;
                \ENDFOR
                \STATE {\bfseries Return:} $Q_T$
             \end{algorithmic}
         \end{algorithm}  
    \end{minipage}\hfill%
    \begin{minipage}[t]{0.53\textwidth}
        \vspace{0pt}
        \begin{algorithm}[H]
            \caption{\adabodeEST}
            \label{alg:density_estpractEST}
            \begin{algorithmic}
                \STATE {\bfseries Input:} distributions $P, Q_0$, $T \in \mathbb{N}_*$,
                $0<\delta\leq 1$;
                \FOR{$t=1$ {\bfseries to} $T$}
                    \STATE $(c_t, \hatmuq{t-1}) \leftarrow \wl(P, Q_{t-1},
                    \delta, T)$;
                    \STATE $\alpha_t \gets \min\cbr{1, \frac{1}{2\csup} \log\rbr{\frac{1+\hatmuq{t-1}}{1-\hatmuq{t-1}}}}$;
                    \STATE Pick $Q_{t}$ as in \eqref{eq:multiplicative_density} with $d_t \defas \exp(c_t)$;
                \ENDFOR
                \STATE {\bfseries Return:} $Q_T$
             \end{algorithmic}
         \end{algorithm}  
    \end{minipage}
\end{figure}

\section{Proofs of formal results}\label{sec:allproofs}

\begin{proposition}\label{prop:recursive_normalisation}
    The normalisation factors can be written recursively with $Z_t = Z_{t-1}·\E_{Q_{t-1}}d_t^{\alpha_t}$.
\end{proposition}
\begin{proof}
    We just need to write
    \begin{gather}
        \mathclap{\frac{Z_{t}}{Z_{t-1}}
        = \frac1{Z_{t-1}}\int \d\tilde Q_t
        = \frac1{Z_{t-1}}\int d_t^{\alpha_t} \d\tilde Q_{t-1}
        =\int d_t^{\alpha_t} \d Q_{t-1} = \E_{Q_{t-1}}d_t^{\alpha_t}}\label{eq:ztzt-1_01}
    \end{gather}
    thus  $Z_t = Z_{t-1}·\E_{Q_{t-1}}d_t^{\alpha_t}$.
\end{proof}

\begin{proposition}\label{prop:qt_exponential_family}
    Let $Q_t$ be defined via \eqref{eq:multiplicative_density} with a sequence of binary classifiers $c_1,\dots,c_t\in\disc $. Then $Q_t$ is an exponential family distribution with natural parameter $\alpha\defas(\alpha_1,\dots,\alpha_t)$ and sufficient statistic $c(x)\defas(c_1(x),\dots,c_t(x))$. 
\end{proposition}

\begin{proof}
    We can convert the binary classifiers $c_1,\dots,c_t\in\disc$ to a sequence of density ratios $(d_i)$ using the connections in \autoref{sec:preliminaries}, which yields
    \begin{gather}
        d_i^{\alpha_i} \defas (\phi\circ\sigma\circ c_i)^{\alpha_i} = \exp\mathinner\circ (\alpha_i c_i).
    \end{gather}
    In this setting, the multiplicative density at round $t$ is
    \begin{align}
        \d Q_t(x)
        &\overset{\eqref{eq:multiplicative_density}}{=} \frac1{\int \prod_{i=1}^t d_i^{\alpha_i} \d Q_0 } \prod_{i=1}^t d_i^{\alpha_i}\d Q_i(x)
        \\&= \exp\rbr{\sum_{i=1}^t\alpha_i c_i(x) - C(\alpha)}\d Q_0(x),\label{eq:exp_mult_dens}
    \end{align}
with $\alpha\defas(\alpha_1,\dots,\alpha_t)$ and 
$C(\alpha ) = \log \int \exp\rbr*{\textstyle\sum_{i=1}^t\alpha_i c_i
}\d Q_0$, which is an exponential family distribution with natural parameter
    $\alpha$, sufficient statistic
    $c(x)\defas(c_1(x),\dots,c_t(x))$, cumulant function $C(\alpha)$, reference measure $Q_0$. We
    note that in the general case, it may be the case that for some non-all-zero
    constants $z_0, z_1, \dots, z_t \in \R, $ we have $z_0 = \sum_{i=1}^t
    z_i c_i(x)$, that is, the representation is not minimal. 
\end{proof}

\begin{lemma}\label{lem:zt_bound_by_kl}
    For any $\alpha_t\in[0,1]$ and $\epsilon_t \in [0,+\infty)^{\cal
      X}$ we have:
    \begin{align}
        \exp\mathopen{}\rbr\Big{ \alpha_t \rbr\big{\E_{Q_{t-1}}\log \epsilon_t-\rkl(P,Q_{t-1})}}\leq  \frac{Z_t}{Z_{t-1}} \leq \rbr{\E_P\epsilon_t}^{\alpha_t}\label{eq:zt_bound_by_kl}
        .
    \end{align}
\end{lemma}
\begin{proof}
    Since $\alpha_t\in[0,1]$, by Jensen's inequality it follows that 
    \begin{gather}
        \E_{Q_{t-1}}d_t^{\alpha_t}
        \leq \rbr{\E_{Q_{t-1}}d_t}^{\alpha_t}
        =\rbr{\int \diff P{Q_{t-1}}·\epsilon_t\d Q_{t-1}}^{\alpha_t}
        = \rbr{\E_P\epsilon_t}^{\alpha_t}.\label{eq:ztzt-1_02}
    \end{gather}
    The upper bound on $Z_t/Z_{t-1}$ follows:
    \begin{gather}
        \frac{Z_{t}}{Z_{t-1}} \overset{\eqref{eq:ztzt-1_01}}{=}\E_{Q_{t-1}} d_t^{\alpha_t} \overset{\eqref{eq:ztzt-1_02}}{≤}  \rbr{\E_P\epsilon_t}^{\alpha_t}.\label{eq:ztzt-1_upper_01}
    \end{gather}

    For the lower bound on ${Z_{t}}/{Z_{t-1}}$, note that 
    \begin{align}
        \log\rbr{\frac{Z_{t}}{Z_{t-1}}} 
       & \overset{\eqref{eq:ztzt-1_01}}{=} \log\E_{Q_{t-1}}d_{t}^\alpha\\
        &\geq \alpha_t\E_{Q_{t-1}}\log d_t\\
        &= \alpha_t\E_{Q_{t-1}}\sbr{\log \epsilon_t+\log\rbr{\frac{\d P}{\d Q_{t-1}}} },
    \end{align}
    which implies the lemma.
\end{proof}

The error term allows us to bound the KL divergence of $P$ from $Q_{t}$ as follows.
\begin{theorem}\label{thm:kl_bound}
    For any $\alpha_t\in[0,1]$, letting $Q_t,Q_{t-1}$ as in \eqref{eq:rec_update},  we have:
    \begin{gather}
        \begin{aligned}
            \forall{d_t \in \radon}
            \kl(P,Q_{t}|_{\alpha_t}) 
            &\leq (1-\alpha_t)\kl(P,Q_{t-1})  + \alpha_t\rbr{\log\E_P\epsilon_t - \E_P\log\epsilon_t}.
        \end{aligned}
        \label{eq:kl_bound2}
    \end{gather}
    where $d_t = \diff P/{Q_{t-1}}·\epsilon_t$.
\end{theorem}
\begin{proof}
    First note that
    \begin{gather}
        \d Q_t = \frac{1}{Z_t}\d \tilde Q_t =\frac{1}{Z_t}d_t^{\alpha_t} \d \tilde Q_{t-1} = \frac{Z_{t-1}}{Z_t}d_t^{\alpha_t} \d Q_{t-1}.\label{eq:prv_dens}
    \end{gather}
    Now consider the following two identities:
    \begin{gather}
        -\alpha_t\log\E_P\epsilon_t\leq\log\rbr{\frac{Z_{t-1}}{Z_t}},\label{eq:zt_bnd}
    \end{gather}
    which follows from \autoref{lem:zt_bound_by_kl}, and
    \begin{align}
        \MoveEqLeft[4]\int\rbr{\log\rbr{\diff P{Q_{t-1}}} -\alpha_t\log d_t}\d P\label{eq:kl_eq}\\
        &=
        \int\rbr{\log\rbr{\diff P{Q_{t-1}}}  -\alpha_t\log\rbr{\diff P{Q_{t-1}}} - \alpha_t\log\epsilon_t}\d P\\
        &=
        (1-\alpha_t)\int\log\rbr{\diff P{Q_{t-1}}}\d P  - \alpha_t\int\log\epsilon_t\d P\\
        &=
        (1-\alpha_t)\kl(P,Q_{t-1})  - \alpha_t\E_P\log\epsilon_t.
    \end{align}
    Then
    \begin{align}
        \kl(P,Q_t)  &= \int \log\rbr{\diff P{Q_{t}}}\d P\\
        &\overset{\eqref{eq:prv_dens}}{=} \int\rbr{\log\rbr{\diff P{Q_{t-1}}} - \log\rbr{\frac{Z_{t-1}}{Z_t}d_t^{\alpha_t}}}\d P\\
        &= \underbrace{\int\rbr{\log\rbr{\diff P{Q_{t-1}}}  -\alpha_t\log d_t}\d P}_{\eqref{eq:kl_eq}}-\underbrace{\log\rbr{\frac{Z_{t-1}}{Z_t}}}_{\eqref{eq:zt_bnd}}\\
        &\leq
        (1-\alpha_t)\kl(P,Q_{t-1})  +\alpha_t\rbr{\log\E_P\epsilon_t -\E_P\log\epsilon_t},
    \end{align}
    as claimed.
\end{proof}

\begin{corollary}\label{cor:qt_rt}
    For any $\alpha_t\in[0,1]$ and $\epsilon_t \in [0,+\infty)^{\cal X}$, letting $Q_t$ as in
    \eqref{eq:multiplicative_density} and $R_t$ from \eqref{eq:R_t_defn}. If $R_t$ satisfies
\begin{gather}
\kl\rbr{P,R_t} \leq \gamma \kl(P,Q_{t-1})\label{assumptionRT}
\end{gather}
for $\gamma \in [0,1]$, then
\begin{gather}
\mathclap{\kl(P,Q_{t}|_{\alpha_t}) \leq
  (1-\alpha_t(1-\gamma))\kl(P,Q_{t-1}).} \label{bGamma1}
\end{gather}
\end{corollary}
\begin{proof} 
    We first show 
    \begin{gather}
       \kl(P,Q_{t}|_{\alpha_t}) \leq (1-\alpha_t)\kl(P,Q_{t-1}) + \alpha_t\kl\rbr{P,R_t}.\label{eqB1R}
    \end{gather}
    By definition $\epsilon_t = \diff {R_t}/P$. The rightmost term in \eqref{eq:kl_bound2} reduces as follows
    \begin{align}
        \log\E_P\epsilon_t -\E_P\log\epsilon_t 
        &= \log \int \diff {\tilde R_t}P \d P - \int \log\rbr{\diff {\tilde R_t}P}\d P
        \\&= \log \int \d \tilde R_t + \int \log\rbr{\diff P{\tilde R_t}}\d P
        \\&= \int \rbr{\log\rbr{\diff P{\tilde R_t}} + \log \int \d \tilde R_t}\d P
        \\&= \int \log\rbr{\diff P{\tilde R_t}·\int \d \tilde R_t}\d P
        \\&= \int \log\rbr{\frac{\d P}{\frac1{\int \d \tilde R_t}\d \tilde R_t}}\d P,
    \end{align}
    which completes the proof of \eqref{eqB1R}. The proof of \eqref{bGamma1} is then immediate.
\end{proof}

Define $\wl$ the weak learner which, taking $P$ and $Q_{t-1}$ as
input, delivers $c_t$ satisfying the conditions of \ref{wla}. In
the boosting theory, which involves a supervised learning problem,
there is one condition instead of two as in \ref{wla}: given a
distribution $D$ over $\cal X \times \cbr{-1, +1}$, we rather require
from the weak learner, \wl$^*$, that
\begin{gather}
    \exists{\gamma \in(0,1]}\frac1{\csup}\E_D y·c_t \geq \gamma,
    \tag{WLA$^*$}\label{wlastar}
\end{gather}
where $y$ denotes the class, mapping $\cal X \to \cbr{-1, +1}$. While is seems rather intuitive that we can craft \wl$^*$ from \wl, it
is perhaps less intuitive as to whether the same can be done for the reverse
direction. We now show that it is indeed the case and \ref{wla} and
\ref{wlastar} are in fact equivalent.

\begin{lemma}\label{lemEQUIVWLA}
    Suppose $\gammap = \gammaq = \gamma$ in \ref{wla} and
    \ref{wlastar}, without loss of generality. Then there exists
    \wl~satisfying \ref{wla} iff there exists \wl$^*$~satisfying \ref{wlastar}. 
\end{lemma} 

\begin{proof}
    To simplify notations, we suppose without loss of generality that
    $\cal C(\cal X) \subseteq [-1,1]^{\cal X}$.

    ($\Rightarrow$) Let $D$ be a distribution on $\cal X \times \{-1,
    +1\}$. It can be factored as a triple $(\pi, P, N)$ where $P$ is a
    distribution over the positive examples, $N$ is a distribution over
    negative examples and $\pi$ is the mixing probability, $\pi \defas
    \Pr_D[y = +1]$. Now, feed $P$ and $N$ in lieu of $P$ and $Q_t$, repectively. We get
    $c_t$ which, from \ref{wla}, satisfies $\E_N[-c_t] \geq \gamma$ and
    $\E_P[c_t] \geq \gamma$, which implies
    \begin{gather}
    \E_D[y c_t] = \pi \E_P[c_t] + (1-\pi) \E_N[-c_t] \geq \pi \gamma +
    (1-\pi)\gamma = \gamma
    \end{gather}
    and we get our weak learner $\wl^*$ satisfying \ref{wlastar}.

    ($\Leftarrow$) We create a two-class classification problem in which observations
    from $P$ have positive class $y = +1$, observations from $Q_t$ have
    negative class $y = -1$ and there is a special observation $x^* \in
    \cal X$ which
    is equally present with probability $1-2\pi$ in both the positive and
    negative class. Hence, we are artificially increasing the difficulty
    of the problem by making its Bayes optimum worse. Obviously, \ref{wlastar}
    having to hold under any distribution, it will have to hold under the
    distribution $D$ that we create. To explicit $D$, consider $\pi \in
    [0, 1/2]$ and the following sampler for $D$:
    \begin{itemize}
    \item sample $z \in [0,1]$ uniformly;
    \begin{itemize}
        \item if $z \leq (1-2\pi)/2$ return $(x^*, +1)$;
        \item else if $z \leq 1-2\pi$ return $(x^*, -1)$;
        \item else if $z \leq 1 - \pi$, return $(x \sim P, +1)$;
        \item else return $(x \sim Q_t, -1)$;
    \end{itemize}
    \end{itemize}
    Let $D$ denote the distribution induced on $\cal X \times \{-1,
    +1\}$. Remark that the error of Bayes optimum is at least $1/2 - \pi$.
    Let $c_t$ returned by \wl$^*$.
    We have because of \ref{wlastar},
    \begin{gather}
        \E_D[y c_t] = \pi (\mup + \muq{t-1}) + \rbr{\frac{1-2\pi}{2} -
    \frac{1-2\pi}{2}}c_t(x^*) = \pi (\mup + \muq{t-1}) \geq \gamma
    \end{gather}
    Consider
    \begin{gather}
    \pi = \frac{\gamma}{1+\gamma}
    \end{gather}
    Which makes 
    \begin{gather}
    \mup + \muq{t-1} \geq 1 + \gamma.
    \end{gather}
    It easily comes that if $\mup < \gamma$, then we must have $\muq{t-1}
    > 1$, which is not possible, and similarly if $\muq{t-1} < \gamma$, then we must have $\mup
    > 1$, which is also impossible. Therefore we have both $\mup \geq
    \gamma$ and $\muq{t-1} \geq \gamma$, and we get our weak learner
    \wl~meeting \ref{wla}, as claimed.
\end{proof}

\subsubsection{Proof of \protect\autoref{thBoostALPHA0}}
\label{sec-proof-thBoostALPHA0}

The proof of \autoref{thBoostALPHA0} is achieved in two steps: (i) any
$c_t$ meeting \ref{wla} can be transformed through scaling into a classifier that
we call \textit{Properly Scaled} without changing it satisfying \ref{wla}
(for the same parameters $\gammap, \gammaq$), (ii)
\autoref{thBoostALPHA0} holds for such Properly Scaled classifiers.
\begin{definition}
    The classifier $c_t$ is said to be Properly Scaled (PS) if it meets: 
    \begin{subequations}
        \makeatletter
            \def\@currentlabel{PS}\label{eq:sdp}
        \makeatother
        \renewcommand{\theequation}{PS.\arabic{equation}}
        \begin{align}
            \exp(2\csup) \leq 2+ \mup\csup \label{sdp1}\\
            \E_{Q_{t-1}} \exp(c_t) \leq \exp\rbr{\frac{\mup \csup}{4}}.\label{sdp2}
        \end{align}
    \end{subequations}
\end{definition}
Hence, we first show how any classifier meeting \ref{wla}
can be made PS \textit{without} changing $\mup$ nor $\muq{t-1}$
(hence, still meeting \ref{wla}), modulo a simple
positive scaling. The proof involves a reverse of Jensen's
inequality which is much simpler than previous bounds \citep{sOAN,sOAU} and
of independent interest.

Our proof will equivalently give upper bounds on $\csup$ that make $c_t$ PS.
We note that our proof is constructive, that is, we give eligible upper bounds for
$\csup$. The proof of \autoref{thPS} is split in several lemmata, the first
of which is straightforward since $\mup\geq 0$ under \ref{wla}. 
\begin{lemma}\label{lemCOND1}
Suppose $c_t$ meets \ref{wla}. Then, \eqref{sdp1} holds for any $\csup \leq \log(2)/2$.
\end{lemma}
To prove how to
satisfy \eqref{sdp2}, we use the notions of Bregman divergences
and Bregman information. For $\varphi : \mathbb{R} \rightarrow \mathbb{R}$ convex differentiable with derivative $\varphi'$, we
define the Bregman divergence with generator $\varphi$ as
$D_\varphi(z\|z') = \varphi(z) - \varphi(z') - (z-z')\varphi'(z')$.
Following \citep{bmdgCW}, we define the \emph{minimal Bregman information} of
$(c_t, Q_{t-1})$ (or just \emph{Bregman information} for short)
relative to $\varphi$ as
\begin{align}
I_\varphi(c_t; Q_{t-1}) \defas \E_{Q_{t-1}} [D_\varphi(c_t \| \E_{Q_{t-1}}
c_t)] .
\end{align}
The Bregman information is a generalization of the variance for which
$\varphi(z) = z^2$. Jensen's inequality would give us a lowerbound,
but we need an \textit{upperbound}. We devise for this objective a
reverse of Jensen's inequality. We suppose that
$c_t$ takes values in $[a, b]$, where we would thus have $|a|$ or $b$
which would be $\csup$. 
\begin{lemma}\label{lemJB}(Reverse of Jensen's inequality)
Suppose $\varphi$ strictly convex differentiable and $c_t(x) \in [a,
b]$ for all $x\in\cal X$. Then,
\begin{align}
I_\varphi(c_t; Q_{t-1}) &\leq
D_{\varphi}\rbr{ u \left\| (\varphi^*)'\rbr{\frac{\varphi(a) - \varphi(b)}{a -
     b} }\right.},\label{eqGenJEn}
\end{align}
where $u$ can be chosen to be $a$ or $b$.
\end{lemma}
\begin{proof}
    The proof is in fact straightforward, as illustrated in \autoref{f-rj}. 
\end{proof}

\begin{figure}[t]
    \begin{center}
        \tikzsetnextfilename{tangent_proof}
        \begin{tikzpicture}[every node/.style={font=\footnotesize}]
            \begin{axis}[clip mode=individual,
                samples=100,
                enlargelimits=false,
                xmin=-0.2, xmax=0.8,
                domain = \pgfkeysvalueof{/pgfplots/xmin}:\pgfkeysvalueof{/pgfplots/xmax},
                ymin=0, ymax=1,
                height=10cm, width=10cm,
                axis lines=middle,
                x axis line style=-,
                y axis line style={draw=none},
                xtick = {0}, ytick = {0},
                enlargelimits=false, clip=false, 
                axis on top,
                after end axis/.code={
                    \path (0,0) node [anchor=north west,yshift=-0.075cm] {} node [anchor=south east,xshift=-0.075cm] {};
                    }]

                \addplot [name path = phi, Blue, thick] {x^2+1/4 - 1/8} coordinate[pos = 0.35] (e) node[right, pos=1] {$\normalsize\phi$};
                \addplot [black, domain = 0.01:0.8, name path=tangent] {0.6974683544 * (x - 0.3487341772) + 0.3716155263 - 1/8};
                
                \coordinate (a) at (-0.04, 0.0016 + 1/4 - 1/8);
                \coordinate (b) at (0.75,  0.8125 - 1/8);
                \coordinate (c) at (0.3487341772,  0.3716155263 - 1/8);
                \coordinate (e') at ([yshift=0.8cm]e) ;
                
                X
                \draw (c) node [draw, circle, fill=black, inner sep=1pt] {};
                \path[name path=stem]  (c) -- (c|-0,1);
                \draw[name path=chord, shorten >= -20pt,  shorten <= -20pt] (a) node [draw, circle, fill=black, inner sep=1pt] {} -- (b) node [draw, circle, fill=black, inner sep=1pt] {};
                
                \addplot [LightBlue] fill between [of=chord and phi, soft clip={domain = -0.04:0.75}];
                
                \draw (e) node [draw, circle, fill=black, inner sep=1pt] {};
                \draw (e') node [draw, circle, fill=black, inner sep=1pt] {} node [above left, yshift=0.4cm, xshift=-0.2cm] {$\E_{Q_{t-1}}\phi(c_t)$} edge[-latex, shorten >= 6pt, shorten <= -4pt, bend right,  in=195] (e'.north west); 
                \draw (e) node [draw, circle, fill=black, inner sep=1pt] {};
                \draw ([xshift = -60pt, yshift=20pt]e) node {$\phi(\E_{Q_{t-1}}c_t)$} edge[-latex, shorten >= 6pt, shorten <= -4pt, bend right,  in=195] (e.north west);
                
                \draw [dashed, name path = bstem] (b) -- (0,0-|b) node [draw, circle, fill=black, inner sep=1pt] {} node [below] {$b$};
                \path [name intersections={of=chord and stem, by=cint}];
                \path [name intersections={of=tangent and bstem, by=bint}];
                \draw (cint) node [draw, circle, fill=black, inner sep=1pt] {};
                \draw [decoration={brace,mirror,raise=4pt},decorate] (c) -- node[yshift=2pt, right=12pt, rotate=90, anchor=center] {RHS} (cint);
                \draw [decoration={brace,mirror,raise=4pt},decorate] (bint) node [draw, circle, fill=black, inner sep=1pt] {}  -- node[yshift=2pt, right=12pt, rotate=90, anchor=center] {RHS} (b);
                \draw [decoration={brace,raise=4pt},decorate] (e) -- node[left=12pt, rotate=90, anchor=center] {LHS} (e');
                \draw [dashed] (cint) -- (0,0-|cint) node [draw, circle, fill=black, inner sep=1pt] {} node [below, xshift=2em] {$\scriptsize(\varphi^*)'\rbr\Big{\frac{\varphi(a) - \varphi(b)}{a -
                b}}$};

                \draw [dashed] (a) -- (0,0-|a) node [draw, circle, fill=black, inner sep=1pt] {} node [below] {$a$};

                \draw [dashed] (e') -- (0,0-|e') node [draw, circle, fill=black, inner sep=1pt] {} node [below, xshift=-1em] {$\E_{Q_{t-1}}c_t$};
            \end{axis}
        \end{tikzpicture}
    \end{center}
    \caption{Proof of \autoref{lemJB}. The tangent represented is
    parallel to the chord linking $(a, \varphi(a))$ and $(b, \varphi(b))$.
    The LHS of \eqref{eqGenJEn} can be represented by the
    vertical difference labelled (always in the blue area). The RHS of
    \eqref{eqGenJEn}, as a Bregman divergence, is the difference between
    the $\phi$ and its tangent at $(\varphi^*)'\rbr*{\frac{\varphi(a) - \varphi(b)}{a-b}}$, measured at either $a$, or $b$ (pictured).}
    \label{f-rj}
\end{figure}
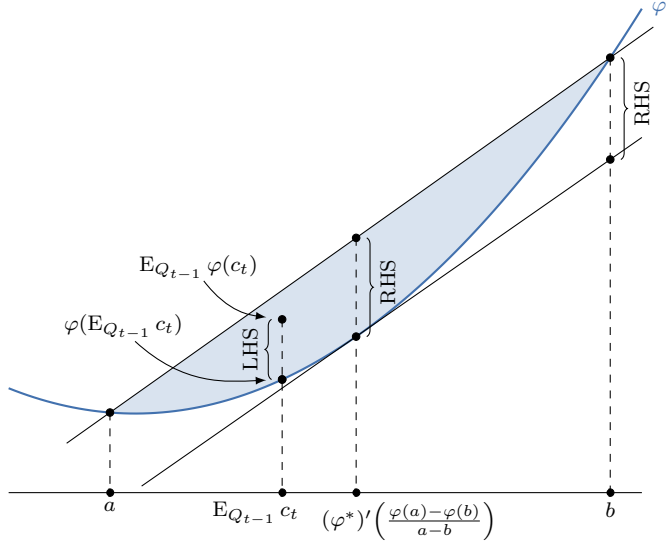

We now show how to satisfy \eqref{sdp2}.
\begin{lemma}\label{threvJ}
    Suppose $c_t$ meets \ref{wla} and 
    \begin{align}
        \csup &\leq \rbr{\frac{\gammap}{4} + \frac{1-\exp(-2\gammaq)}{2}}.
    \end{align}
    Then
    \begin{align}
        \E_{Q_{t-1}} \exp(c_t) &\leq \exp\rbr{\frac{\mup \csup}{4}},
    \end{align}
    that is, \eqref{sdp2} holds.
\end{lemma}
\begin{proof}
    Consider $\varphi(z) = \exp(z)$ and so $\varphi^*(z)
    = z \log z - z$ in \autoref{lemJB}. Suppose without loss of generality that $a = -\csup,
    b = \csup$. We get
    \begin{align}
        I_{\exp}(c_t; Q_{t-1}) &= \E_{Q_{t-1}}
        \exp(c_t) - \exp \E_{Q_{t-1}} c_t \\
        &\leq D_{\varphi}\rbr{ \csup \left\| (\varphi^*)'\rbr{\frac{\varphi(\csup) - \varphi(-\csup)}{2\csup} }\right.}.
    \end{align} 
    Now, we just need
    to ensure that 
    \begin{align}
    D_{\varphi}\rbr{ \csup \left\| (\varphi^*)'\rbr{\frac{\varphi(\csup) - \varphi(-\csup)}{2\csup} }\right.} &\leq \exp\rbr{ \frac{\mup\csup}{4}} - \exp(-\muq{t-1} \csup),\label{eqTSH}
    \end{align}
    as indeed we shall then have, because of \ref{wla},
    \begin{align}
        \E_{Q_{t-1}}
        \exp(c_t) &\leq \exp\rbr{ \frac{\gammap\csup}{4}} -
        \exp(-\gammaq \csup) + \exp \E_{Q_{t-1}}
        c_t \\
        &\leq \exp\rbr{ \frac{\mup\csup}{4}} -\exp (-\muq{t-1}
        \csup) +\exp (-\muq{t-1} \csup) \\
        &= \exp\rbr{ \frac{\mup\csup}{4}} ,
    \end{align}
    which is the statement of the lemma. 

\begin{proposition}\label{bregbound}
Pick $\varphi \defas \exp$. If $|z|\leq 2$, then
\begin{gather}
    D_{\varphi}\rbr{ z \left\|
        (\varphi^*)'\rbr{\frac{\varphi(z) -
            \varphi(-z)}{2z} }\right.}\leq {z}^2.
\end{gather}
\end{proposition}
\begin{proof}
    Equivalently, we need to show
    \begin{align}
        z^2 &\geq \exp(z) \rbr{\frac{1}{2} - \frac{1}{2z}} + \exp(-z)
                \rbr{\frac{1}{2} + \frac{1}{2z}}\\
                &\hspace{4em}+\rbr{\frac{\exp(z)-\exp(-z)}{2z}} \log\rbr{\frac{\exp(z)-\exp(-z)}{2z}}.
    \end{align}
    We split the proof in two. First, let us fix
    \begin{gather}
    g_1(z) \defas \frac{2 \rbr{ \exp(z) \rbr{\frac{1}{2} - \frac{1}{2z}} + \exp(-z)
    \rbr{\frac{1}{2} + \frac{1}{2z}}}}{z^2}.
    \end{gather}
    We remark that
    \begin{gather}
    g'_1(z) = \frac{\exp(-z)
    (-z^2  - 3 z - 3+\exp(2 z) (z^2 - 3 z + 3))}{2 z^4}.
    \end{gather}
    We then remark that, letting $g_2(z) \defas -z^2  - 3 z - 3+\exp(2 z) (z^2 - 3 z + 3)$,
    \begin{align}
    g_2(z) & = -z^2  - 3 z - 3 + \sum_{k\geq 0} \frac{2^k z^{k+2}}{k!} -
    \sum_{k\geq 0} \frac{3 \cdot 2^k z^{k+1}}{k!} + \sum_{k\geq 0}
    \frac{3\cdot 2^k z^{k}}{k!}\\
    & = -z^2  - 3 z - 3 + \sum_{k\geq 2} \frac{2^{k-2} z^{k}}{(k-2)!} -
    \sum_{k\geq 1} \frac{3 \cdot 2^{k-1} z^{k}}{(k-1)!} + \sum_{k\geq 0}
    \frac{3\cdot 2^k z^{k}}{k!}\\
    & = \sum_{k\geq 2} \rbr{\frac{2^{k-2}}{(k-2)!}-\frac{3\cdot 2^{k-1}}{(k-1)!}+\frac{3\cdot 2^{k}}{k!}} z^k \\
    & = \sum_{k\geq 5} \rbr{\frac{2^{k-2}}{(k-2)!}-\frac{3\cdot 2^{k-1}}{(k-1)!}+\frac{3\cdot 2^{k}}{k!}} z^k \\
    & = \sum_{k\geq 5} \rbr{k^2 - 7k + 12} \frac{2^{k-2} z^k}{k!}.
    \end{align}
    We then check that $z\mapsto z^2-7z+12 < 0$ only for $z \in (3,4)$.
    That is, it is never negative over naturals so $g_2(z)\geq 0,
    \forall z\geq 0$. We also check that $\lim_0 g'_1(z) = 0$ and so
    $g_1(z)$ is increasing for $z\geq 0$. Finally,
    \begin{gather}
    g_1(2) = \frac{1}{2}\cdot \rbr{\frac{\exp(2)}{4} +
    \frac{3}{4\exp(2)}} < \frac{7.81}{8} < 1,
    \end{gather}
    which shows that
    \begin{gather}
        \forall{|z|\leq 2}\exp(z) \rbr{\frac{1}{2} - \frac{1}{2z}} + \exp(-z)
        \rbr{\frac{1}{2} + \frac{1}{2z}} \leq \frac{z^2}{2}.\label{bpart1}
    \end{gather}
    (The analysis for $z<0$ uses the fact that the function is even.) We
    now show that we have
    \begin{gather}
        \forall{z \in [-2,2]}\frac{\exp(z) -\exp(-z)}{2z} \leq 1 + \frac{z^2}{4}.\label{bpart2}
    \end{gather}
    Hence, we want to show that $\exp(z) \leq \exp(-z) + 2z + z^3/2$ for
    $z \in [-2,2]$. We now have $\exp(-z) \geq 1 - z + z^2/2 - z^3/6 +
    z^4/24 - z^5/120$ for
    $z\geq 0$, so
    we just need to show $\exp(z) \leq 1 - z + z^2/2 - z^3/6 +
    z^4/24 - z^5/120  + 2z + z^3/2 = 1 + z + z^2/2 + z^3/3 +
    z^4/24 - z^5/120  + 2z + z^3/2$ for
    $z \in [0,2]$ (we will then use the fact that both functions in
    \eqref{bpart2} are even), which simplifies, using Taylor series for $\exp$, in showing
    \begin{gather}
        \forall{z \in [0,2]}\sum_{k\geq 6} \frac{z^{k}}{k!} \leq \frac{z^3}{6} - \frac{z^5}{60},
    \end{gather}
    or after dividing both sides by $z^3 > 0$ (the inequality is obviously
    true for $z=0$),
    \begin{gather}
        \forall{z \in (0,2]} \sum_{k\geq 6} \frac{1}{k(k-1)(k-2)} \cdot \frac{z^{k-3}}{(k-3)!} \leq \frac{1}{6} - \frac{z^2}{60},
    \end{gather}
    Since $k(k-1)(k-2)\geq 120$ for $k\geq 6$, it is
    enough to show that $\sum_{k\geq 6} \frac{z^{k-3}}{(k-3)!} \leq 20 -
    2z^2$. But $\sum_{k\geq 6} \frac{z^{k-3}}{(k-3)!} = \sum_{k\geq 3}
    \frac{z^{k}}{k!} = \exp(z) - 1 - z - z^2/2$, so we just need to
    show that $\exp(z) \leq 21 + z - 3z^2/2$ for $z \in (0,2]$, which
    is easy to show as the rhs is concave, decreasing for $z\geq 1/3$ and intersecting
    $\exp$ for $z\geq 5/2$. So \eqref{bpart2} holds. Since $\log(z) \leq z
    - 1$, we get
    \begin{gather}
    \rbr{\frac{\exp(z)-\exp(-z)}{2z}} \log
    \rbr{\frac{\exp(z)-\exp(-z)}{2z}} \leq \frac{z\exp(z)-z\exp(-z)}{8}
    \end{gather}
    Now, we have $\exp(z)  - \exp(-z) - 4z \leq 0$ for $z\in
    [0,2]$, since the function is strictly convex for $z\geq 0$ with two
    roots at $z=0$ and $z>2$. Reorganising, this shows that
    $(z\exp(z)-z\exp(-z))/8 \leq z^2/2$ for $z\in
    [0,2]$, and so
    \begin{gather}
        \forall{z \in [0,2]}\rbr{\frac{\exp(z)-\exp(-z)}{2z}} \log
    \rbr{\frac{\exp(z)-\exp(-z)}{2z}} \leq \frac{z^2}{2}.
    \end{gather}
    Putting it together with \eqref{bpart1}, we now have
    \begin{align}
        \MoveEqLeft[2] D_{\varphi}\rbr{ z \left\|(\varphi^*)'\rbr{\frac{\varphi(z) -\varphi(-z)}{2z} }\right.}
            \\&= \exp(z) \rbr{\frac{1}{2} - \frac{1}{2z}} + \exp(-z)
            \rbr{\frac{1}{2} + \frac{1}{2z}}
            \\&\hspace{4em}+\rbr{\frac{\exp(z)-\exp(-z)}{2z}} \log\rbr{\frac{\exp(z)-\exp(-z)}{2z}}
            \\&= \frac{z^2}{2} +\frac{z^2}{2} 
            \\&= z^2
    \end{align}
    for $z\in [0,2]$, and therefore, since both functions are even, the
    same holds for $z\in [-2, 0]$ and completes the proof.
\end{proof}
To show \eqref{eqTSH}, we therefore just need to ensure $\csup$ small enough so
    that 
    \begin{align}
    {\csup}^2 &\leq \exp\rbr{ \frac{\mup\csup}{4}} - \exp(-\muq{t-1} \csup)\label{constcsup2}.
    \end{align}
    Because $\exp(-\muq{t-1} \csup)$ is convex, it is upper-bounded
    over the interval $[0,2]$ by its chord between
    its two points in abscissae 0 and 2,
    \begin{align}
        \forall{\csup \in [0, 2]}\exp(-\muq{t-1} \csup) &\leq 1 - \frac{1-\exp(-2\muq{t-1})}{2}\csup,
    \end{align}
    and we also have, since $\exp(z) \geq 1+ z$,
    \begin{align}
    \exp\rbr{ \frac{\mup\csup}{4}} & \geq 1 + \frac{\mup\csup}{4}.
    \end{align}
    To ensure \eqref{constcsup2}, it is therefore sufficient, as long as
    $\csup \in (0, 2]$, that
    \begin{align}
    {\csup}^2 &\leq \rbr{\frac{\mup}{4} + \frac{1-\exp(-2
    \muq{t-1})}{2}}\csup,
    \end{align}
    which, after simplification and considering \ref{wla}, is achieved
    provided 
    \begin{align}
    {\csup} &\leq \rbr{\frac{\gammap}{4} + \frac{1-\exp(-2
    \gammaq)}{2}}.\label{bsupcstar2}
    \end{align}
    There remains to check that the condition of \autoref{bregbound} applies, that is, $\csup \leq 2$. The maximal
    value of the rhs in \eqref{bsupcstar2}, taking into
    account that $\gammap, \gammaq\leq 1$, is $1/4 + (1-\exp(-2))/2
    \approx 0.57 < 2$, which shows that the condition of \autoref{bregbound} indeed applies and proves \autoref{threvJ}.
    \end{proof}

    \begin{theorem}\label{thPS}
        Suppose $c_t$ satisfies \ref{wla}. Then there exists a constant $\eta > 0$ such
        that $\eta \cdot c_t$ satisfies \ref{wla} and is PS.
    \end{theorem}    

    \begin{proof}
    Even when better bounds are possible, the combination of
    \autoref{lemCOND1} and \autoref{threvJ} show that any $c_t$ satisfying the
    \ref{wla}, positively scaled so that $\csup \leq \log(2)/2$, still satisfies
    \ref{wla} \textit{and} is PS, as claimed.
    \end{proof}
    We shall now prove \autoref{thBoostALPHA0}. The proof mainly consists of two lemmata, one showing that
    $\E_{Q_{t-1}} \exp\rbr{\alpha_t c_t}$ is small, the second one
    showing, under conditions on $c_t$, that $\E_{Q_{t-1}} \exp\rbr{c_t}$ is
    conveniently upper-bounded by $\E_{Q_{t-1}} \exp\rbr{\alpha_t c_t}$,
    leading to the theorem.
    \begin{lemma}\label{lemWLA} 
        Let $ \alpha_t \defas  \frac{1}{2\csup} \log \rbr{\frac{1+\muq{t-1}}{1-\muq{t-1}}}$. Then 
    \begin{gather}
        \E_{Q_{t-1}} \exp\rbr{\alpha_t c_t} \leq \sqrt{1-\muq{t-1}^2}.
    \end{gather}
    \end{lemma}
    \begin{proof}
    We know  \citep{nnAR} that
    \begin{gather}
        \forall{a,b \in [-1,1]} 
        1-ab \geq \sqrt{1-a^2} \exp\rbr{-\frac{b}{2}\log\rbr{\frac{1+a}{1-a}}}.\label{eq:wacky_rn_inequality}
    \end{gather}
    Let $a \defas \muq{t-1} $ and $b \defas-c_t / \csup$, for short. Then we obtain
    \begin{align}
        \exp\rbr{\alpha_t c_t} &= \exp\rbr{ -\rbr{-c_t \cdot \frac{1}{2\csup} \log\rbr{\frac{1 + \muq{t-1}}{1 - \muq{t-1}}}}}
        \\&\overset{\eqref{eq:wacky_rn_inequality}}{\leq}\frac{1+\muq{t-1}\frac{c_t}{\csup}}{\sqrt{1-\muq{t-1}^2}}
        \\&=
        \frac{1-\muq{t-1}\cdot\rbr{-\frac{c_t}{\csup}}}{\sqrt{1-\muq{t-1}^2}},\label{eq:bEXP1}
    \end{align}
    which implies the lemma.
\end{proof}

\begin{lemma}\label{lemBIA}
    Fix any $J\geq 0$. Suppose that the two conditions hold:
    \begin{align}
    \E_{Q_{t-1}} \exp(c_t) & \leq \exp\rbr{\frac{J}{2}},\label{condJJ1}\\
    \\\muq{t-1} & \leq \frac{J}{1+J}. \label{condJJ2}
    \end{align}
    Then,
        \begin{gather}
            \E_{Q_{t-1}} \exp(c_t)  \leq  \frac{1}{\sqrt{1-\muq{t-1}^2}}\cdot \E_{Q_{t-1}} \exp\rbr{\alpha_t c_t + J}.
        \end{gather}
\end{lemma}
\begin{proof}
    Jensen's inequality yields
    \begin{gather}
        \E_{Q_{t-1}} \exp (\alpha_t c_t)  \geq  \exp \rbr{\E_{Q_{t-1}} \alpha_t c_t} = \exp\rbr{- \alpha_t \csup \muq{t-1}},
    \end{gather}
    hence we rather show the stronger statement
    \begin{align}
        \E_{Q_{t-1}} \exp\rbr{c_t} \leq \frac{1}{\sqrt{1-\muq{t-1}^2}}\cdot\exp\rbr{ -\alpha_t \csup \muq{t-1} + J}.
    \end{align}
    We use two inequalities:
    \begin{gather}
        \forall{z \in[0,1]}
        \frac{2z^2}{1-z} \geq 4 \log\frac{1}{\sqrt{1-z^2}} \geq z \log\rbr{\frac{1+z}{1-z}}.\label{eqZ2}
    \end{gather}
    Let us summarize these as $A\geq B\geq C$. To first check these
    inequalities, we remark:
\begin{itemize}
\item to check $A\geq B$, we simplify it: it
    yields equivalently $g_1(z) \defas z^2(1+z) \geq
    -(1-z^2)\log(1-z^2) \defas g_2(z)$. We then
    check that $g'_2(z) = 2z(1+\log(1-z^2))$ while $g'_1(z) = 2z(1+
    3z/2)$. Both derivatives are continuous with the same limit in $0$
    and it is easy to check that for $z\geq 0$, $g'_2(z) \leq
    g'_1(z)$. Since $g_1(0) = g_2(0)$, we get $A\geq B$;
\item to check $B \geq C$, we simplify it, which yields equivalently
  $g_3(z) \defas (z-2)\log(1-z) - (z+2) \log(1+z) \geq 0$. We have
  $g''_3(z) = 4z^2/(z^2-1)^2 \geq 0$, which shows the strict convexity
  of the function. We also have $g'_3(0) = g_3(0) = 0$, which gives
  $g_3(z) \geq 0$ for all $z$ and shows $B\geq C$.
\end{itemize}
With the latter ineq. \eqref{eqZ2} and the expression of $\alpha_t$
for the regular boosting regime, we get
    \begin{align}
        \MoveEqLeft[4]\frac{1}{\sqrt{1-\muq{t-1}^2}}\cdot\exp\rbr{-\alpha_t \csup \muq{t-1} + J}
        \\& = \exp\rbr{J + \log\rbr\Bigg{\frac{1}{\sqrt{1-\muq{t-1}^2}}} -
        \frac{\muq{t-1}}{2}\log\rbr\Bigg{\frac{1+\muq{t-1}}{1-\muq{t-1}}}} \label{chain1}
        \\& \geq \exp\rbr{J-
        \frac{\muq{t-1}}{4}\log\rbr{\frac{1+\muq{t-1}}{1-\muq{t-1}}}}
    \\& \geq \exp\rbr{J - \frac{1}{4} \cdot \frac{2
        \muq{t-1}^2}{1-\muq{t-1}}}. \label{chain2}
    \end{align}
    The last inequality follows from the former
    ineq. \eqref{eqZ2}. Suppose now that we can ensure
    \begin{align}
    \frac{2\muq{t-1}^2}{1-\muq{t-1}} &\leq 2J\label{condJ1}.
    \end{align}
    It would follow from \eqref{chain2} that 
    \begin{align}
        \frac{1}{\sqrt{1-\muq{t-1}^2}}\cdot\exp\rbr{-\alpha_t \csup
        \muq{t-1} + J} & \geq \exp\rbr{J - \frac{1}{4} \cdot 2J}
    = \exp\rbr{\frac{J}{2}},
    \end{align}
    and so to prove the lemma, we would just need 
    \begin{gather}
        \E_{Q_{t-1}} \exp(c_t) \leq \exp\rbr{\frac{J}{2}},
    \end{gather}
    which is precisely \eqref{condJJ1}. To get \eqref{condJ1}, we equivalently need
    $\muq{t-1}^2 + J \muq{t-1} - J\leq 0$, that is, 
\begin{gather}
\muq{t-1} \leq
    \frac{1}{2}\cdot(-J+ \sqrt{J^2+4J})\label{condJAF}
\end{gather} 
To prove a simpler
    equivalent condition, we let $g_4(z) \defas
    (1+z)\sqrt{z^2+4z}/(z(3+z))$. We easily get $\lim_{\downarrow 0}
    g_4(z) = +\infty$, $\lim_{+\infty} g_4(z)
    = 1$ and $g'_4(z) = - 6(z+4)/N$ with $N \defas
    (z^2+4z)^{3/2}(3+z)^2 \geq 0$, so $g_4(z)\geq 1$ for all $z\geq 0$, 
    and reordering this inequality yields equivalently $z/(1+z) \leq
    (1/2)\cdot(-z+ \sqrt{z^2+4z})$ for $z\geq 0$, so to get
    \eqref{condJAF}, we just require $\muq{t-1} \leq J/(1+J)$, which
    is \eqref{condJJ2}, and ends the proof of \autoref{lemBIA}.
\end{proof}

Let $\alpha_t \defas \min\cbr{1, \frac{1}{2\csup}\log\rbr{ \frac{1+\muq{t-1}}{1-\muq{t-1}} }  }$. Because there are two regimes for $\alpha_t$,
we define two boosting regimes, a \emph{high boosting regime}, $\alpha_t =
1$ (``clamped''), and a \emph{regular boosting regime}, $\alpha_t <
1$ (``not clamped''). We show two rates of decrease for the KL divergence, one for each regime.

\paragraph{Convergence in the regular boosting regime} The \ref{wla} alone is sufficient to guarantee a significant decrease of the KL divergence of $P$ from $Q_{t-1}$ at each boosting iteration. The proof of the theorem uses a simple reverse of Jensen's
inequality which may be of independent interest. Note that even when
we require that $c_t$ meet \ref{wla}, the decrease of the $\kl$
divergence uses its \textit{actual} values for $\mup, \muq{t-1}$,
which can yield a substantially larger $\kl$ decrease.

\begin{theorem}\label{thBoostALPHA0}
    In the regular boosting regime and under \ref{wla}, 
    \begin{align}
      \mathclap{\kl\rbr{P,Q_{t}|_{\alpha_t}}   \leq  \kl(P,Q_{t-1}) -\frac{\mup}{4} \log \rbr{\frac{1+\muq{t-1}}{1-\muq{t-1}}}.}\label{bKLRBR}
    \end{align}
\end{theorem}
\begin{proof}
    We have
    \begin{gather}
        \E_P \epsilon_t = \E_{Q_{t-1}} \sbr{\frac{\d P}{\d Q_{t-1}} \cdot \epsilon_t} = \E_{Q_{t-1}} \exp\rbr{c_t}.\label{bEPST}
    \end{gather}
    Hence, combining successively the statements of \autoref{lemBIA}
    (we check below that the conditions of the lemma are indeed satisfied)
    and \autoref{lemWLA}, we get:
    \begin{align}
        \log \E_P \epsilon_t & = \log \E_{Q_{t-1}} \exp\rbr{c_t} \\
 & \leq  \log\rbr{\frac{1}{\sqrt{1-\muq{t-1}^2}}\cdot\E_{Q_{t-1}} \exp\rbr{ \alpha_t c_t + J}}\label{chain5}\\
        & =  \log\rbr{\frac{\E_{Q_{t-1}} \exp\rbr{\alpha_t c_t}}{\sqrt{1-\muq{t-1}^2}}\cdot \exp\rbr{J}}
        \\
        & \leq \log \exp\rbr{J}\\
        &= J. \label{chain6}
    \end{align}
    On the other hand, \ref{wla} yields 
    \begin{align}
        \mup \csup = \E_P c_t   =  \E_P \log \rbr{\frac{\d P}{\d Q_{t-1}} \cdot\epsilon_t}
        = \kl(P,Q_{t-1}) + \E_P \log \epsilon_t.\label{mucstar}
    \end{align}
    Since $\alpha_t \geq 0$, it follows from \autoref{thm:kl_bound}
    and \eqref{mucstar}, \eqref{chain6} in this order that
    \begin{align}
        \kl(P,Q_{t}) 
        & \leq \kl(P,Q_{t-1})  - \alpha_t\rbr{\kl(P,Q_{t-1}) +
        \E_P\log\epsilon_t}+ \alpha_t \log\E_P\epsilon_t \label{chain3}
        \\& = \kl(P,Q_{t-1})  - \alpha_t\mup \csup + \alpha_t \log\E_P\epsilon_t
        \\& \leq \kl(P,Q_{t-1})  - \alpha_t\mup \csup + \alpha_t J
    \\& = \kl(P,Q_{t-1})  - \alpha_t(\mup \csup -  J). \label{chain4}
    \end{align}
    It remains to fix $J \defas \mup \csup/2$, and we get 
    \begin{align}
        \kl(P,Q_{t}) 
        & \leq \kl(P,Q_{t-1})  - \frac{\alpha_t \mup \csup}{2}\label{balphalowboosting}
    \\& = \kl(P,Q_{t-1})  - \frac{\mup}{4} \log \rbr{\frac{1+\muq{t-1}}{1-\muq{t-1}}},
    \end{align}
    which is the statement of the theorem. We end up the proof of \autoref{thBoostALPHA0} by
    showing that the PS property for $c_t$ implies that the conditions of
    \autoref{lemBIA} are satisfied --- hence, Theorem
    \ref{thBoostALPHA0} is shown for $c_t$ being PS, which we recall
    is always possible from \autoref{thPS} when $c_t$ satisfies the
    \ref{wla}. While it is clear that \eqref{condJJ1} is one of the PS
    properties for $c_t$, we still need to show that the PS ensures
    \eqref{condJJ2} with $J = \mup \csup/2$, that is, we need to show that
    \begin{align}
    \muq{t-1} &\leq \frac{\mup \csup}{2  +\mup \csup}.\label{bmuq1}
    \end{align}
    Recall that we are in the regular boosting regime where we do not clamp $\alpha_t$, and
    therefore, if we let
    \begin{align}
    \mucsup \defas \frac{\exp(2\csup)-1}{\exp(2\csup)+1} \in (0,1),\label{defMUCSUP}
    \end{align}
    then we know that $\muq{t-1} \leq \mucsup$, so to have \eqref{bmuq1},
    it suffices to ensure $\mucsup \leq \mup \csup / (2  +\mup \csup)$,
    which equivalently yields
    \begin{align}
    \exp(2\csup) &\leq 2 + \mup \csup,
    \end{align}
    which is the first PS property. This ends the proof of \autoref{thBoostALPHA0}.
\end{proof}

\subsubsection{Proof of \autoref{thBoostALPHA2}}\label{sec-proof-thBoostALPHA2}
\paragraph{Convergence in the high boosting regime} This is where
things get interesting; when $\alpha_t$ is clamped to 1, the decrease in the
KL divergence at each iteration is \textit{guaranteed} to be of
order $\csup$, and can even be significantly larger depending on the
actual values of $\muq{t-1}$ and $\mucsup$, defined as
\begin{gather}
    \mucsup \defas \frac{\exp(2\csup)-1}{\exp(2\csup)+1} \in (0,1).\label{defMUCSUPMAIN}
\end{gather}
Because $\alpha_t = 1$, we have $\muq{t-1} \geq \mucsup$, so
let us write $\muq{t-1} = (1+\delta_{t-1}) \mucsup$ for some
$\delta_{t-1}\geq 0$. Note that \autoref{thBoostALPHA2} doese not assume \ref{wla}. It is worthwhile remarking that \protect\autoref{rateWLA} is a direct
consequence of \protect\autoref{thBoostALPHA0} above.

We follow some of the same steps as for \autoref{thBoostALPHA0}. 
\begin{lemma}
        Let $ \alpha_t \defas  1$. Then
    \begin{gather}
        \E_{Q_{t-1}}
        \exp\rbr{c_t} \leq \frac{1-\muq{t-1}\mucsup}{\sqrt{1-\mucsup^2}},
    \end{gather}
    where $\mucsup$ is defined in \eqref{defMUCSUP}.
\end{lemma}
\begin{proof}
We have this time $\E_{Q_{t-1}}
    \exp\rbr{c_t}  = \E_{Q_{t-1}}
    \exp\rbr{\alpha_t c_t} $.
We use again \eqref{eq:wacky_rn_inequality} with $a = \mucsup$ and
get, instead of \eqref{eq:bEXP1}:
\begin{align}
    \exp\rbr{\alpha_t c_t} & \leq \frac{1-\mucsup\cdot\rbr{-\frac{c_t}{\csup}}}{\sqrt{1-\mucsup^2}},
\end{align}
which implies the lemma after taking the expectation and remarking
that for the choice $a = \mucsup, \alpha_t = 1$.
\end{proof}

\begin{theorem}\label{thBoostALPHA2}
    In the high boosting regime,
    \begin{align}
        \kl\rbr{P,Q_{t}|_{\alpha_t}}  
        &\leq \kl(P,Q_{t-1}) - \mup \csup - \mucsup^2 \cdot \smash{\rbr\Bigg{\frac{1}{2} + \frac{\delta_{t-1}}{1-\mucsup^2}}}.
    \end{align}
\end{theorem}

\begin{proof}
    Since we get a direct bound on $\E_{Q_{t-1}}
    \exp\rbr{c_t} $, we can achieve the proof of \autoref{thBoostALPHA2} via \eqref{bEPST} and \eqref{chain3} as 
    \begin{align}
        \kl(P,Q_{t}) 
        & \leq \kl(P,Q_{t-1})  - \alpha_t\rbr{\kl(P,Q_{t-1}) + \E_P\log\epsilon_t}+ \alpha_t \log\E_P\epsilon_t
        \\& \leq \kl(P,Q_{t-1})  - \mup \csup + \log\E_P\epsilon_t
        \\& \leq \kl(P,Q_{t-1})  -  \mup \csup + \log \frac{1-\muq{t-1} \mucsup}{\sqrt{1-\mucsup^2}}
        \\& = \kl(P,Q_{t-1})  - \mup \csup 
        \\&\hspace{6em}+ \log\rbr{\frac{1-\mucsup^2}{\sqrt{1-\mucsup^2}} -(\muq{t-1}-\mucsup)\cdot \frac{\mucsup}{\sqrt{1-\mucsup^2}}}
        \\& = \kl(P,Q_{t-1})  - \mup \csup 
        \\&\hspace{6em}+ \log\rbr{\sqrt{1-\mucsup^2} -(\muq{t-1}-\mucsup)\cdot\frac{\mucsup}{\sqrt{1-\mucsup^2}}}
        \\& = \kl(P,Q_{t-1})  - \mup \csup +\frac{1}{2}\cdot \log (1-\mucsup^2) 
        \\&\hspace{6em}+\log\rbr{ 1 -(\muq{t-1}-\mucsup)\cdot\frac{\mucsup}{1-\mucsup^2}}
        \\& \leq \kl(P,Q_{t-1})  - \mup \csup - \frac{\mucsup^2}{2} - (\muq{t-1}-\mucsup)\cdot
            \frac{\mucsup}{1-\mucsup^2}\label{blog1}
        \\& \leq \kl(P,Q_{t-1})  - \mup \csup - \mucsup^2 \cdot
    \rbr{\frac{1}{2} + \frac{\delta_{t-1}}{1-\mucsup^2}},
    \end{align}
    where we have let $\muq{t-1} = (1+\delta_{t-1}) \mucsup$. In
    \eqref{blog1}, we have used $\log(1-x)\leq -x$. 
\end{proof}
I

\begin{theorem}\label{rateWLA}
    Suppose \ref{wla} holds at each iteration. Then using $Q_t$ as in
    \eqref{eq:multiplicative_density} and $\alpha_t$ as in \autoref{ssec:convergence_under_weak_learning_assumptions}, we are
    guaranteed that $\kl(P,Q_{T}) \leq \varrho$ after a number of
    iterations $T$ satisfying: 
\begin{gather}
T \geq 2 \cdot \frac{\kl(P,Q_{0}) - \varrho}{\gammap\gammaq}.
\end{gather}
\end{theorem}

\begin{proof}
    The proof stems from the regular boosting regime, using
    $\log((1+z)/(1-z))\geq 2z$ for $z\geq 0$. Better rates are possible
    using the high boosting regime, and in any case, $Q_t$ as in
    \eqref{eq:multiplicative_density} and $\alpha_t$ as in \autoref{ssec:convergence_under_weak_learning_assumptions} define a simple
    boosting algorithm to come up with an analytical expression for $Q_T$
    that provably converges to $P$. 
\end{proof}

\subsection{Proof of \protect\autoref{geomBOOST}}
\label{sec-proof-geomBOOST}

We reformulate the theorem involving a new notation for readability
purpose in the proof.
\begin{theorem}\label{geomBOOST}
    Suppose \ref{wla} holds with $\gammaq \defas \gammar ·\csup$ for some
    $\gammar > 0$ and \ref{wda} hold at each boosting iteration. Then we
    get after $T$ boosting iterations:
    \begin{align}
        \mathclap{\kl(P,Q_{T}) \leq 
        \rbr{1 - \frac{\min\{2,
            \gammar\}\gammap}{2(1+\gammae)}}^T \cdot \kl(P,Q_{0}).}
    \end{align}
\end{theorem}
\begin{proof}
    We proceed in two steps, first showing how \ref{wda} bounds
    $\kl(P,Q_{t-1})$. We have by definition $\log(\d P / \d Q_{t-1}) +  \log \epsilon_t =
    c_t \leq \csup$, and so, taking expectations, we get $\kl(P,Q_{t-1}) +
    \csup \mu_{\epsilon_t} \leq \int \d P \csup = \csup$. Hence, 
    \begin{align}
    \kl(P,Q_{t-1})
    &\leq \csup - \csup \mu_{\epsilon_t} \leq (1+\gammae)\csup.\label{boundKL1}
    \end{align}
    We now show the statement of the theorem. Suppose we are in the low-boosting regime where $\alpha_t$ is not
    clamped. In this case, since $\log((1+z)/(1-z))\geq 2z$, we have
    \begin{align}
    \alpha_t \geq \frac{\muq{t-1}}{\csup} \geq \gammar,
    \end{align} 
    and it comes from \eqref{balphalowboosting} 
    \begin{align}
        \kl(P,Q_{t}) 
        & \leq \kl(P,Q_{t-1})  - \frac{\gammar \gammap \csup}{2}\label{balphalowboosting2}.
    \end{align}
    In the high-boosting regime, we have immediately $\kl(P,Q_{t}) \leq
    \kl(P,Q_{t-1})  - \gammap \csup$. So, letting $\rho \defas
    \min\cbr{1, \gammar/2}$, we get under the assumptions of the theorem $\kl(P,Q_{t}) \leq
    \kl(P,Q_{t-1})  - \rho \gammap \csup$, and \ref{wda} yields in addition through
    \ref{boundKL1}, 
    \begin{align}
    \kl(P,Q_{t}) &\leq
    \kl(P,Q_{t-1})  - \frac{\rho\gammap}{1+\gammae} \cdot \kl(P,Q_{t-1}) \\
 & = \rbr{1 - \frac{\min\cbr{1, \gammar/2}\gammap}{1+\gammae} } \cdot \kl(P,Q_{t-1})\\
 & = \rbr{1 - \frac{\min\cbr{2, \gammar}\gammap}{2(1+\gammae)} } \cdot \kl(P,Q_{t-1}),
    \end{align}
    and we get the statement of the theorem by replacing $\gammar$ by
    its expression. This ends the proof of  \protect{\autoref{geomBOOST2}}.
\end{proof}

\subsubsection{Proof of \protect{\autoref{thBoostALPHA0PLUS}}}\label{sec-proof-thBoostALPHA0PLUS}

The proof of \autoref{thBoostALPHA0PLUS} is essentially a rewriting of
the proof of \autoref{thBoostALPHA0} and \autoref{thBoostALPHA2},
taking into account that we have just samples from distributions to
compute the estimates of edges and \ref{wla}. We split the proof in three
steps, one that provides an additional Lemma we shall need for the
next steps, one for the non-clamped regime for $\alpha_t$, one for the
clamped regime for $\alpha_t$.

\noindent \textbf{Step.1}, We need the additional simple Lemma, which is an
exploitation of basic concentration inequalities \citep[Section 3.1]{mdC}.
\begin{lemma}\label{lemCONC}
For any $0<\delta \leq 1$ and $0<\kappa\leq 1$, suppose the weak learner samples at each
iteration $t = 1, 2, ..., T$, $m_P$ times $P$ and $m_Q$ times $Q_t$,
such that the following constraints hold:
\begin{gather}
    m_P  \geq  \frac{1}{\kappa^2 {\gammap}^2} \log \frac{4T}{\delta} \and
    m_Q  \geq \frac{1}{\kappa^2 {\gammaq}^2} \log \frac{4T}{\delta}.
\end{gather} 
Then there is probability $\geq 1 - \delta$ that for any $t = 1, 2,
..., T$, the current estimators $\hatmup$ of $\mup$ and $\hatmuq{t-1}$ of
$\muq{t-1}$ satisfy:
\begin{gather}
    |\hatmup - \mup| \leq \kappa \gammap, \label{EMUP}\\
    |\hatmuq{t-1} - \muq{t-1}| \leq \kappa \gammaq.\label{EMUQ}
\end{gather}
\end{lemma}
From now on, we denote as $E$ the proposition that both \eqref{EMUP} and
\eqref{EMUQ} hold for all $T$ iterations, for some $0<\kappa\leq 1$
that will be computed later.

We have a slightly weaker version of \autoref{lemWLA}, straightforward to prove from \autoref{lemWLA}.

\begin{lemma}\label{lemWLA2}
   Let $ \alpha_t \defas  \frac{1}{2\csup} \log
   \rbr{\frac{1+\hatmuq{t-1}}{1-\hatmuq{t-1}}}$. Then we have under $E$,
\begin{gather}
    \E_{Q_{t-1}} \exp\rbr{\alpha_t c_t} \leq \sqrt{1-\hatmuq{t-1}^2} +
    \frac{\kappa \gammaq \hatmuq{t-1}}{\sqrt{1-\hatmuq{t-1}^2}}.
\end{gather}
\end{lemma}
\begin{lemma}\label{lemBIA2}
Fix any $J\geq 0$. Suppose that the two conditions hold:
\begin{align}
\E_{Q_{t-1}} \exp(c_t) & \leq \exp\rbr{\frac{J}{2}},\label{condJJ12}
\\ \hatmuq{t-1} & \leq \frac{J}{1+J}. \label{condJJ22}
\end{align}
Then we have under $E$,
    \begin{gather}
        \mathclap{\E_{Q_{t-1}} \exp(c_t)  \leq
        \frac{1}{\sqrt{1-\hatmuq{t-1}^2}}\cdot \E_{Q_{t-1}}
        \exp\rbr{\alpha_t c_t + J} \cdot \exp\rbr{\frac{\kappa \gammaq}{2}\log\rbr\Bigg{\frac{1+\hatmuq{t-1}}{1-\hatmuq{t-1}}}}.}
    \end{gather}
\end{lemma}
\begin{proof}
Because the proof mixes the use of $\hatmuq{t-1}$ and $\muq{t-1}$, we
re-sketch the major lines of the proof from \autoref{lemBIA}. First,
Jensen's inequality still yields $\E_{Q_{t-1}} \exp (\alpha_t c_t)
\geq  \exp\rbr{- \alpha_t \csup \muq{t-1}}$, so we in fact prove
\begin{align}
    \frac{1}{\sqrt{1-\hatmuq{t-1}^2}}\cdot\exp\rbr{ -\alpha_t \csup \muq{t-1} + J} \geq \E_{Q_{t-1}} \exp\rbr{c_t}.
\end{align}
The chain of (in)equalities in \eqref{chain1}--\eqref{chain2} now
becomes withe the use of $E$:
\begin{align}
    \MoveEqLeft[4]\frac{1}{\sqrt{1-\hatmuq{t-1}^2}}\cdot\exp\rbr{-\alpha_t \csup \muq{t-1} + J}
        \\& = \exp\rbr{J +
        \log\rbr\Bigg{\frac{1}{\sqrt{1-\hatmuq{t-1}^2}}} -
        \frac{\muq{t-1}}{2}\log\rbr\Bigg{\frac{1+\hatmuq{t-1}}{1-\hatmuq{t-1}}}}
    \\& \geq \exp\Biggl(J +
        \log\rbr\Bigg{\frac{1}{\sqrt{1-\hatmuq{t-1}^2}}} -
        \frac{\hatmuq{t-1}}{2}\log\rbr\Bigg{\frac{1+\hatmuq{t-1}}{1-\hatmuq{t-1}}}
    \\&\hspace{4em}   
        - \frac{\kappa \gammaq}{2}\log\rbr\Bigg{\frac{1+\hatmuq{t-1}}{1-\hatmuq{t-1}}}
        \Biggr)
    \\& \geq \exp\rbr{J - \frac{1}{4} \cdot \frac{2
        \hatmuq{t-1}^2}{1-\hatmuq{t-1}} - \frac{\kappa \gammaq}{2}\log\rbr\Bigg{\frac{1+\hatmuq{t-1}}{1-\hatmuq{t-1}}}}.
\end{align}
Provided we have $\hatmuq{t-1} \leq J/(1+J)$, which is
\eqref{condJJ22}, we have similarly to \autoref{lemBIA},
\begin{align}
\frac{2\hatmuq{t-1}^2}{1-\hatmuq{t-1}} &\leq 2J\label{condJ12}.
\end{align}
Hence, it follows that 
\begin{align}
    \MoveEqLeft[2]\exp\rbr{\frac{\kappa \gammaq}{2}\log\rbr\Bigg{\frac{1+\hatmuq{t-1}}{1-\hatmuq{t-1}}}}
    \cdot \frac{1}{\sqrt{1-\muq{t-1}^2}}\cdot\exp\rbr{-\alpha_t \csup
      \muq{t-1} + J} 
    \\& \geq \exp\rbr{J - \frac{1}{4} \cdot 2J}
    \\& = \exp\rbr{\frac{J}{2}},
\end{align}
and so to prove the lemma, we would just need 
\begin{gather}
    \E_{Q_{t-1}} \exp(c_t) \leq \exp\rbr{\frac{J}{2}},
\end{gather}
which is \eqref{condJJ12}.
\end{proof}
Now, instead of \eqref{chain5}--\eqref{chain6}, we get
\begin{align}
    \log \E_P \epsilon_t & \leq
    \log\rbr{\frac{1}{\sqrt{1-\muq{t-1}^2}}\cdot\E_{Q_{t-1}} \exp\rbr{
        \alpha_t c_t + J}} +
    \frac{\kappa \gammaq}{2}\log\rbr\Bigg{\frac{1+\hatmuq{t-1}}{1-\hatmuq{t-1}}}
   \\
    & =  \log\rbr{\frac{\E_{Q_{t-1}} \exp\rbr{\alpha_t
          c_t}}{\sqrt{1-\muq{t-1}^2}}\cdot \exp\rbr{J}} +
    \frac{\kappa \gammaq}{2}\log\rbr\Bigg{\frac{1+\hatmuq{t-1}}{1-\hatmuq{t-1}}}
    \\
    & \leq \log \rbr{\rbr{1 + \frac{\kappa \gammaq
          \hatmuq{t-1}}{1-\hatmuq{t-1}^2}}\exp \rbr{J}} +
    \frac{\gammaq}{4}\log\rbr\Bigg{\frac{1+\hatmuq{t-1}}{1-\hatmuq{t-1}}}\\
    &= J + \log \rbr{1 + \frac{\kappa\gammaq
          \hatmuq{t-1}}{1-\hatmuq{t-1}^2}} + \frac{\kappa\gammaq}{2}\log\rbr\Bigg{\frac{1+\hatmuq{t-1}}{1-\hatmuq{t-1}}}. 
\end{align}
We get from \eqref{chain4}
\begin{align}
    \kl(P,Q_{t}) 
    & \leq \kl(P,Q_{t-1})  - \alpha_t\rbr{\mup \csup -  J - J'} \label{chain42}
\end{align}
with, because $\log(1+x)\leq x$,
\begin{align}
J' &\defas \log \rbr{1 + \frac{\kappa\gammaq
          \hatmuq{t-1}}{1-\hatmuq{t-1}^2}} +
      \frac{\kappa\gammaq}{2}\log\rbr\Bigg{\frac{1+\hatmuq{t-1}}{1-\hatmuq{t-1}}}\\
 &= \log \rbr{1 + \frac{\kappa\gammaq
          \hatmuq{t-1}}{1-\hatmuq{t-1}^2}} +
      \frac{\kappa\gammaq}{2}\log\rbr\Bigg{1+\frac{2\hatmuq{t-1}}{1-\hatmuq{t-1}}}\\
 &\leq \frac{\kappa\gammaq
          \hatmuq{t-1}}{1-\hatmuq{t-1}^2} + \frac{\kappa\gammaq\hatmuq{t-1}}{1-\hatmuq{t-1}}\\
 &\leq \kappa \cdot \frac{2\gammaq\hatmuq{t-1}}{1-\hatmuq{t-1}}
\end{align}
Now, we would like from the PS property and \eqref{bmuq1} that we have:
\begin{align}
\hatmuq{t-1} \leq \frac{\mup \csup}{2  +\mup \csup},\label{bmuq12}
\end{align}
so
\begin{align}
J' &\leq \kappa \gammaq \mup \csup,
\end{align}
and we get from \eqref{chain42},
\begin{align}
\kl(P,Q_{t}) 
    & \leq \kl(P,Q_{t-1})  - \alpha_t\rbr{(1-\kappa \gammaq)\mup \csup -  J} ,\label{chain43}
\end{align}
and if we fix again $J = \mup \csup/2$, we get this time
\begin{align}
    \kl(P,Q_{t}) 
    & \leq \kl(P,Q_{t-1})  - \alpha_t  \cdot \rbr{\frac{1}{2} - \kappa \gammaq} \cdot \mup \csup.
\end{align}
If we pick $\kappa$ satisfying
\begin{align}
\kappa & \leq \min \cbr{1, \frac{1}{4\gammaq}}, \label{constKAPPA1}
\end{align}
then
we are guaranteed $1/2 - \kappa \gammaq \geq 1/4$ and so
\begin{align}
    \kl(P,Q_{t}) 
    & \leq \kl(P,Q_{t-1})  - \frac{\mup}{8} \log \rbr{\frac{1+\hatmuq{t-1}}{1-\hatmuq{t-1}}},\label{leqKLMU}
\end{align}
In the same way as for \autoref{thBoostALPHA0}, we
ensure \eqref{bmuq12} by noting that, since we are in the case where we do not clamp $\alpha_t$, letting
\begin{align}
\hatmucsup \defas \frac{\exp(2\csup)-1}{\exp(2\csup)+1} \in (0,1),
\end{align}
then we again need to ensure $\mucsup \leq \mup \csup / (2  +\mup
\csup)$, which again yields to the first PS property. 

We are not yet done as we now have to replace $\mup$ by its estimate,
$\hatmup$, in \eqref{leqKLMU}. Under $E$, we obtain 
\begin{align}
    \kl(P,Q_{t}) 
    & \leq \kl(P,Q_{t-1})  - \frac{\hatmup - \kappa \gammap}{8} \log \rbr{\frac{1+\hatmuq{t-1}}{1-\hatmuq{t-1}}},
\end{align}
and under the (EWLA), we know that $\hatmup \geq \gammap$, so if we also
put the constraint $\kappa \leq 1/2$, then $\kappa \gammap \leq
\gammap/2 \leq \hatmup/2$ and so:
\begin{align}
    \kl(P,Q_{t}) 
    & \leq \kl(P,Q_{t-1})  - \frac{\hatmup}{16} \log \rbr{\frac{1+\hatmuq{t-1}}{1-\hatmuq{t-1}}},
\end{align}
as claimed. This ends the proof of Step.2 by remarking
two additional facts: (i) we have not changed the PS properties, and
(ii) 
we have two constraints over $\kappa$ (also considering
\eqref{constKAPPA1}), which can be both satisfied by choosing (since
$\gammaq \leq 1$) $\kappa$ satisfying
\begin{align}
\kappa & \leq
\frac{1}{4}. \label{constKAPPA2}
\end{align}

\begin{theorem}\label{thBoostALPHA0PLUS}
    Suppose \ref{ewla} holds. Then with probability of at least $1 -
    \delta$, 
    \begin{align}
        \forall{t = 1, 2, ..., T}\kl(P,Q_{t})  \leq \kl(P,Q_{t-1})  - \Delta_t,\label{leqKLMUMAIN}
    \end{align}
    where 
    \begin{gather}
        \Delta_t \defas \begin{cases}
            \frac{\hatmup}{16} \log \rbr{\frac{1+\hatmuq{t-1}}{1-\hatmuq{t-1}}} & \text{in the non-clamped regime,}\\
            \frac{\hatmup \csup}{2} + \mucsup^2 \cdot \rbr{\frac{1}{4} + \frac{\hat{\delta}_{t-1} }{1-\mucsup^2}} & \text{otherwise.}
        \end{cases}
    \end{gather}
\end{theorem}

\begin{proof}
    We proceed in exactly the same way as we did for
    \autoref{thBoostALPHA2}. We first remark that \autoref{lemWLA2}
    \textit{is still valid} in this case, so that we still have
    \begin{gather}
        \E_{Q_{t-1}}
        \exp\rbr{c_t} \leq \frac{1-\muq{t-1}\mucsup}{\sqrt{1-\mucsup^2}}.
    \end{gather}
    It is not hard to check that we then keep the exact same derivations
    as for \autoref{thBoostALPHA2}, yielding 
    \begin{align}
        \kl(P,Q_{t}) 
        & \leq \kl(P,Q_{t-1})  - \mup \csup - \mucsup^2 \cdot
    \rbr{\frac{1}{2} + \frac{\delta_{t-1}}{1-\mucsup^2}},
    \end{align}
    where we have let $\muq{t-1} = (1+\delta_{t-1}) \mucsup$. Remark that
    this time, $\delta_{t-1}$ is not necessarily positive since we do not
    have access to $\muq{t-1}$ --- this may happen when $\muq{t-1} < \hatmuq{t-1}$.
    What we do, to finish up
    Step.3, is replace $\delta_{t-1}$ by the $\hat{\delta}_{t-1}$
    for which we have $\hatmuq{t-1} = (1+\hat{\delta}_{t-1}) \mucsup$,
    which we are then sure is going to satisfy $\hat{\delta}_{t-1}\geq 0$
    under the clamped regime for $\alpha_t$. Under $E$, we have
    \begin{align}
    \delta_{t-1} &= \frac{\muq{t-1}}{\mucsup} - 1\\
    & \geq \frac{\hatmuq{t-1}}{\mucsup} - 1 - \kappa
    \cdot \frac{\gammaq}{\mucsup}\\
    &  = \hat{\delta}_{t-1} - \kappa
    \cdot \frac{\gammaq}{\mucsup}
    \end{align}
    yielding
    \begin{align}
        \kl(P,Q_{t}) 
        & \leq \kl(P,Q_{t-1})  - \mup \csup - \mucsup^2 \cdot
    \rbr{\frac{1}{2} - \kappa
    \cdot \frac{\gammaq}{\mucsup(1-\mucsup^2)} + \frac{\hat{\delta}_{t-1} }{1-\mucsup^2}},
    \end{align}
    Suppose we pick $\kappa$ such that
    \begin{align}
    \kappa & \leq \frac{\mucsup(1-\mucsup)}{2}\label{constKAPPA3}.
    \end{align}
    Since $\mucsup \in [0,1]$, we also have
    \begin{align}
    \kappa & \leq \frac{\mucsup(1-\mucsup^2)}{2}\label{constKAPPA32}.
    \end{align}
    In this case, we obtain, since $\gammaq \leq 1$,
    \begin{align}
        \kl(P,Q_{t}) 
        & \leq \kl(P,Q_{t-1})  - \mup \csup - \mucsup^2 \cdot
    \rbr{\frac{1}{4} + \frac{\hat{\delta}_{t-1} }{1-\mucsup^2}}.
    \end{align}
    Finally, we also know under $E$ that $\mup \csup \geq \hatmup\csup -
    \kappa \gammap \csup$. Under the (EWLA), we know that $\hatmup \geq \gammap$, so if we again
    put the constraint $\kappa \leq 1/2$ (satisfied from \eqref{constKAPPA2}), then $\kappa \gammap \csup\leq
    \gammap \csup/2 \leq \hatmup \csup/2$ and so:
    \begin{align}
        \kl(P,Q_{t}) 
        & \leq \kl(P,Q_{t-1})  - \frac{\hatmup \csup}{2} - \mucsup^2 \cdot
    \rbr{\frac{1}{4} + \frac{\hat{\delta}_{t-1} }{1-\mucsup^2}},\label{leqKLM2}
    \end{align}
    which ends the proof of Step.3 once we remark that \eqref{constKAPPA2}
    and \eqref{constKAPPA3} are both satisfied if
    \begin{align}
    \kappa & = \min\cbr{\frac{1}{4},\frac{\mucsup(1-\mucsup)}{2}}
    = \frac{\mucsup(1-\mucsup)}{2} = \kappasup. \label{constKAPPAFIN}
    \end{align}
\end{proof}

\section{Experimental procedure}\label{sec:experimental_setup}

All models were trained using the ADAM optimiser with the defult settings from \textsc{Flux.jl} \citep{innes:2018}: $\eta = 0.001$, $\beta_1 = 0.9$, $\beta_2 = 0.999$, $\epsilon = 1e-08$. 
In all experiments we divide the data into training (75\%) and test (25\%) sets, which we use to early stop on certain experiments. 
The reset of the experimental conditions are presented in \autoref{tab:experimental_setup}. Each experiment was run 20 times.

\begin{sidewaystable}
    \caption{Experimental procedure\label{tab:experimental_setup}}
    \begin{minipage}{\textwidth}
        \footnotesize
    \makebox[\textwidth][c]{
        \begin{tabular}{c p{3cm} p{1.5cm} p{1.5cm} c c c c c c c}
            \toprule
            Experiment                                  & $P$                                                                               & $Q_0$                           & $\alpha_t$        &\begin{tabular}[x]{@{}c@{}} Sample size\\ ($P,Q$)\end{tabular}\
                                                                                                                                                                                                              &Epochs
                                                                                                                                                                                                                   &Batch size
                                                                                                                                                                                                                        &Early stop\footnote{This parameter terminates training when the the test error falls this amount below the training error. Useful to stabalise training that might otherwise fail due to exploding test error.}
                                                                                                                                                                                                                            & Network topolgoy of $c_t$ 
            \\\midrule
            \autoref{ssec:error_and_convergence},
            \autoref{ssec:activation_functions}         & 8 mode Gaussian ring mixture $\sigma=1$                                           & Isotropic Gaussian $\sigma = 1$ & 1/2               &(1000,1000)&3000& 50& Not used
            & $\cal X \xrightarrow[\textrm{dense}]{\text{ReLU}}\R^{5}\xrightarrow[\textrm{dense}]{\text{ReLU}}\R^{5}\xrightarrow[\textrm{dense}]{\text{ReLU}}\R$
            \\
            \autoref{ssec:architecture_comparison_nll}  & 8 mode Gaussian ring mixture $\sigma=1$                                           & Isotropic Gaussian $\sigma = 1$ & 1/2               &(1000,1000)&3000& 50& Not used
            & varies
            \\
            \autoref{ssec:convergence_across_dimensions}& randomly arranged 8 mode Gaussian mixture                                         & Isotropic Gaussian $\sigma = 1$ & 1/2               &(1000,1000)&2000& 250& 20\%
            & $\cal X \xrightarrow[\textrm{dense}]{\text{ReLU}}\R^{10}\xrightarrow[\textrm{dense}]{\text{ReLU}}\R^{10}\xrightarrow[\textrm{dense}]{\text{ReLU}}\R$
            \\
            \autoref{ssec:comparison_with_kde}          & randomly arranged 8 mode Gaussian mixture                                         & Empirically fit Gaussian & Selected to minimise NLL &(1000,1000)&3000& 50& Not used
            & $\cal X \xrightarrow[\textrm{dense}]{\text{ReLU}}\R^{10}\xrightarrow[\textrm{dense}]{\text{ReLU}}\R^{10}\xrightarrow[\textrm{dense}]{\text{ReLU}}\R$
            \\
            \autoref{ssec:comparison_with_adagan}       & \textsc{Adagan} generated randomly arranged 8 mode Gaussian mixture               & Empirically fit Gaussian & Selected to minimise NLL &(5000,5000)\footnote{\textsc{AdaGAN} trains on a set of (64000,64000) samples, we take a subset of these to use for training $Q_t$.}&1000& 250& 3\% 
            & $\cal X \xrightarrow[\textrm{dense}]{\text{ReLU}}\R^{10}\xrightarrow[\textrm{dense}]{\text{ReLU}}\R^{10}\xrightarrow[\textrm{dense}]{\text{ReLU}}\R$
            \\\bottomrule
        \end{tabular}
    }\end{minipage}
\end{sidewaystable}

\end{document}